\documentclass{article}

\IfFileExists{iclr2027_conference.sty}{%
  \usepackage{iclr2027_conference,times}%
}{%
  \usepackage[margin=1in]{geometry}%
  \usepackage{natbib}%
  \typeout{^^JNOTE: iclr2027_conference.sty not found; generic layout used.^^J}%
}
\iclrfinalcopy   
\usepackage[utf8]{inputenc}
\usepackage[T1]{fontenc}
\usepackage{microtype}
\usepackage{hyperref}
\usepackage{url}
\usepackage{xcolor}

\usepackage{graphicx}
\usepackage{wrapfig}
\usepackage{subcaption}
\usepackage{booktabs}
\usepackage{nicefrac}
\usepackage{multirow}

\usepackage{amsmath}
\usepackage{amssymb}
\usepackage{amsfonts}
\usepackage{mathtools}
\usepackage{amsthm}
\usepackage{siunitx}

\usepackage{enumitem}
\usepackage[capitalize,noabbrev]{cleveref}

\newlist{conditions}{enumerate}{1}
\setlist[conditions]{leftmargin=*,align=left,labelsep=0.5em,topsep=-\parskip+2pt,partopsep=0pt,parsep=0pt,itemsep=0pt}

\usepackage{amsmath,amssymb,amsthm}
\usepackage{booktabs}
\usepackage{float}
\usepackage{array}
\newcolumntype{L}[1]{>{\raggedright\arraybackslash}p{#1}}
\usepackage{graphicx}
\usepackage{xcolor}
\usepackage{hyperref}
\usepackage{url}
\usepackage{mathrsfs}
\usepackage{todonotes}
\usepackage{xspace}
\usepackage{enumitem}

\usepackage{aliascnt}

\newtheorem{theorem}{Theorem}[section]
\newaliascnt{proposition}{theorem}
\newtheorem{proposition}[proposition]{Proposition}
\aliascntresetthe{proposition}
\newaliascnt{lemma}{theorem}
\newtheorem{lemma}[lemma]{Lemma}
\aliascntresetthe{lemma}
\newaliascnt{corollary}{theorem}
\newtheorem{corollary}[corollary]{Corollary}
\aliascntresetthe{corollary}
\newaliascnt{definition}{theorem}
\newtheorem{definition}[definition]{Definition}
\aliascntresetthe{definition}
\newaliascnt{example}{theorem}

\aliascntresetthe{example}
\newaliascnt{remark}{theorem}
\newtheorem{remark}[remark]{Remark}
\aliascntresetthe{remark}
\crefname{proposition}{Proposition}{Propositions}
\crefname{lemma}{Lemma}{Lemmas}
\crefname{corollary}{Corollary}{Corollaries}
\crefname{definition}{Definition}{Definitions}
\crefname{example}{Example}{Examples}
\crefname{remark}{Remark}{Remarks}

\newcommand{\R}{\mathbb{R}}

\DeclareMathOperator{\dist}{dist}
\newcommand{\Tr}{\mathsf{T}}

\usepackage{xparse}

\NewDocumentCommand{\cellname}{o}{%
  \IfValueTF{#1}
    {the Neural Finite State Machine (NFSM)}
    {NFSM\xspace}%
}
\newcommand{\condstyle}[1]{#1}
\makeatletter
\newcommand{\condlabel}[1]{\def\@currentlabel{#1}\label{cond:#1}}
\makeatother

\usepackage{tikz}
\usetikzlibrary{arrows.meta,calc,positioning,shapes.geometric,shapes.misc}

\usepackage{pgfplots}
\usepackage{pgfplotstable}
\usepgfplotslibrary{fillbetween}
\pgfplotsset{compat=1.18}
\definecolor{cellA}{RGB}{206,228,247}
\definecolor{cellB}{RGB}{206,236,214}
\definecolor{cellC}{RGB}{252,229,199}
\definecolor{cellD}{RGB}{245,211,229}
\definecolor{cellE}{RGB}{228,225,243}
\definecolor{codecol}{RGB}{160,35,35}
\definecolor{deccol}{RGB}{20,90,150}
\definecolor{failcol}{RGB}{225,30,30}

\definecolor{cfp64}{RGB}{31,119,180}
\definecolor{cfp32}{RGB}{255,127,14}
\definecolor{cbf16}{RGB}{44,160,44}
\definecolor{cfp16}{RGB}{214,39,40}
\definecolor{cfpa}{RGB}{148,103,189}
\definecolor{cfpb}{RGB}{140,86,75}

\newcommand{\vaat}{VAA}

\newcommand{\cH}{\mathcal{H}}
\newcommand{\cO}{\mathcal{O}}
\newcommand{\cA}{\mathcal{A}}
\newcommand{\cD}{\mathcal{D}}
\newcommand{\cS}{\mathcal{S}}

\newcommand{\readout}{\psi}
\newcommand{\hstar}{h^{\star}}
\newcommand{\sep}{\mathrm{sep}}
\newcommand{\cR}{\mathcal{R}}
\newcommand{\cW}{\mathcal{W}}
\newcommand{\idle}{\Phi_{\bar{x}}}

\title{Train Short, Act Long: Multistability and the Horizon Independence of Recurrent Policies}

\author{%
  Asad Bakija, Florent De Geeter, Julien Brandoit, Pierre Sacré \& Guillaume Drion \\
  Montefiore Institute, University of Liège \\
  Liège, Belgium \\
  \texttt{\{Asad.Bakija,gdrion\}@uliege.be} \\
}

\begin{document}

\maketitle
\begin{abstract}
In partially observable reinforcement learning (RL), a decision often depends on a cue seen many steps earlier, and training at such horizons is costly.
Ideally, a policy trained on short horizons would remain optimal at arbitrarily long ones, making the training horizon a matter of convenience.
We formalize this property as \emph{delay-horizon independence} (DHI): optimality at every horizon, at a memory, computation and precision cost that is bounded in the horizon. We study agents that carry the cue in a recurrent memory of fixed size, whose cost does not grow with the delay.
This paper characterizes DHI through the dynamics of that memory, in theory and on benchmarks.
Although quantified over all horizons, DHI reduces to a property of a single map: a recurrent policy is DHI if and only if it admits a continuous readout of the memory, invariant under the idle dynamics and separating the cues. 
Monostable memory then precludes DHI at any parameter values, whereas multistable memory enables it. Monostability covers all contracting affine recurrences, hence state space models and gated linear recurrent neural networks (RNNs).
Under arbitrarily small perturbations of the idle dynamics, multistability is the only mechanism supporting DHI at bounded memory.
On the T-maze, the effective horizon of monostable models resists hyperparameter search: selected horizons fall short by up to $745\times$ when retrained. 
Multistable models trained on short horizons stay optimal up to $10^6$, on the T-maze and on MiniGrid.
Current parallelizable RNNs thus trade horizon independence for scalability: parallelizable multistable architectures are the direction this points to.
\end{abstract}





\section{Introduction}

In long-horizon partially observable Markov decision processes (POMDPs), an informative cue and the decision that depends on it are separated by many uninformative steps. The agent must carry the cue across that interval, and both the exploration cost and the sample inefficiency grow with its length. One way around that cost is to \emph{train short and act long}: to train on a short version of the task and deploy on a long one. This only helps if the policy keeps working at the deployment horizon. Generalization of this kind is difficult in reinforcement learning (RL): policies overfit to their training environment \citep{korkmaz2024surveyanalyzinggeneralizationdeep} and become unpredictable in unusual situations, even under full observability \citep{ghosh2021generalizationrldifficultepistemic}. \citet{myers2025horizongeneralizationreinforcementlearning} introduced \emph{horizon generalization}, the ability to learn from short horizons and generalize to long ones, for fully observed goal-conditioned MDPs. Their goal is visible at every step, so their policies need no memory and cannot forget.
 
Our setting, however, differs in two ways. First, the cue is presented once and must be carried over steps, so the agent can forget it. Forgetting is the entire difficulty here. When the goal is visible at every step, the agent has nothing to forget. Second, a guarantee that covers a finite range of horizons is usable only when its boundary is known. For a memory that fades, that boundary can be found only by testing at the deployment horizon. Fading-memory models do generalize over a finite range of horizons, but cannot be independent of the horizon in both the decision and the cost (\cref{section-mono}). The question is therefore whether a policy validated at short horizons can be trusted at \emph{every} horizon, at a cost that does not grow with the delay, without ever being tested at the deployment horizon.
 
To answer it, we formalize \emph{delay-horizon independence} (DHI) for a class of delayed-cue POMDPs: optimality at every horizon, at a memory, computation and precision cost bounded in the horizon, which excludes solutions that buffer the history. Several mechanisms can carry the cue across the delay. Attention over a stored history does it at a cost that grows with the delay. A recurrent memory of fixed size does it at a cost that stays constant. We focus on the second approach because the property requires a bounded cost. This restriction to a fixed-size memory enables us to frame the process as a dynamical system.
 
This paper characterizes DHI through the dynamics of that memory, classifies the memory of common recurrent architectures accordingly, and tests the classification on standard POMDP benchmarks. Related work is discussed in \cref{appendix-related-work}. Our main contributions are as follows.
 
\textbf{DHI reduces to a property of one map.}  We prove that DHI holds if and only if the policy admits a continuous readout of the memory that is invariant under the idle dynamics and separates the cues.
 
\textbf{Monostability precludes DHI, multistability enables it.} A monostable memory, with a single global attractor, erases the separation between cues and admits no horizon-independent policy for any parameters or training procedure. This is checkable in closed form for every affine recurrence, and thus excludes state space models (SSMs) \citep{gu2021combiningrecurrentconvolutionalcontinuoustime,gu2022efficientlymodelinglongsequences,gu2024mambalineartimesequencemodeling} and gated linear recurrent neural networks (RNNs) \citep{martin2018parallelizinglinearrecurrentneural,qin2023hierarchicallygatedrecurrentneural,beck2024xlstmextendedlongshortterm,feng2024weregrudminimalgatedrecurrent,zhu2025scalablematmulfreelanguagemodeling}: selectivity and gating only slow the decay. A multistable memory supports DHI as long as its state never crosses a basin boundary. Under arbitrarily small perturbations of the idle dynamics, it is the only mechanism that supports DHI at bounded memory.
 
\textbf{The effective horizon of monostable memory is not a design parameter.} On a T-maze benchmark, every monostable model, nonlinear or parallelizable, is indeed optimal only up to a finite effective horizon $T^\star$. This horizon spans more than five orders of magnitude across a hyperparameter search, but a configuration selected for a long horizon does not reproduce it when retrained from different seeds, falling short by up to $745\times$. Multistable models trained on short horizons remain optimal up to $10^6$ steps across cell types, and the bistable memory recurrent unit (BMRU) \citep{degeeter2026parallelizable}, parallelizable and free of transient dynamics, points to multistability as the cause: every BMRU configuration that solves the task stays optimal up to the evaluation ceiling. The same contrast holds on MiniGrid-Memory and MiniGrid-GoToDoor, evaluated at grid sizes up to $10^6$.

\section{Delay-horizon independence}
\label{section-DHI}

We first formalize the setting  and class of decision processes  under study (\cref{section-setting}), then define delay-horizon independence (\cref{section-thidef}). 

\subsection{Delayed-cue POMDPs}
\label{section-setting}

We consider a POMDP $\mathcal{M} = (\cS, \cA, \cO, p_0, P, \Omega, R, \gamma)$ with discount factor $\gamma \in (0,1]$, in the notation of \cref{appendix-pomdp}.
The observation $o_t \in \cO$ only partially
reveals the state $s_t \in \cS$, so acting optimally generally requires the entire past
history of observations and actions. 

We study agents that compress that history into a memory of fixed size.
Formally, the memory state $h_t \in \cH$, equipped with a norm $\|\cdot\|$, is updated by a recurrent equation:
\begin{equation}
    h_t = \Phi(h_{t-1}, x_t), \qquad x_t \coloneqq (o_t, a_{t-1}),
  \label{eq:update}
\end{equation}
where $o_t \in \cO$ is the current observation and $a_{t-1} \in \cA$ is the previous
action. The agent then selects $a_t = \pi(h_t)$, where $\pi$ is a policy. We call a policy that uses a memory state updated by \cref{eq:update} to map past observations and actions to a current action a \emph{recurrent policy}. 
The map $\Phi$, with $\Phi_x \coloneqq \Phi(\cdot, x)$, is realized by an RNN with learnable parameters $\theta$, 
so that recurrent policies in this sense cover the standard architectures for POMDPs, from classical nonlinear RNNs to more recent parallelizable linear RNNs and SSMs. Update equations, steady-state derivations and architectural details
for every model we discuss are collected in \cref{appendix-RNNdyn}. A policy is \emph{optimal} for $\mathcal{M}$ if it maximizes the expected return over all history-dependent policies (\cref{appendix-pomdp}).


We analyze POMDPs in which an informative cue is separated from the decision that depends on it by an uninformative phase, which we call the idle phase.

\newpage
\begin{definition}[Delayed-cue POMDP]
  \label{def:class}
  A family $\{\mathcal{M}_T\}_{T \geq T_{\min}}$ of POMDPs, with $T_{\min} \geq 1$ and sharing the observation and action spaces, is a \emph{delayed-cue family with minimum horizon $T_{\min} \geq 1$} if, for every horizon $T \geq T_{\min}$:
  \begin{conditions}[label={\normalfont\condstyle{(A\arabic*)}},ref={(A\arabic*)}]        
    \item\label{a:cue} \emph{Cue:} At $t = 0$, the agent receives a key observation $\sigma$ from a finite set of cues $\Sigma \subseteq \cO$.
    \item\label{a:idle} \emph{Uninformative observations:} For every $1 \leq t \leq T$, the conditional law of $o_t$ given the past actions is independent of $\sigma$.
    \item\label{a:dec} \emph{Decision step:} At $t = T$, the optimal action is $a^*(\sigma, o_T)$, and it does not depend on $T$.
    \item\label{a:reach} \emph{Reachability:} From every state reachable during the idle phase, some sequence of actions reaches the decision step.
  \end{conditions}
\end{definition}

Taken together, these conditions say that the cue must be retained, and nothing more.
\ref{a:cue} is a normalization: a cue arriving later is handled by re-anchoring $t=0$ at its arrival. 
\ref{a:idle} and \ref{a:dec} say that $\sigma$ carries \emph{information} that is necessary at the decision step and not recoverable during the idle phase, so a policy that has not retained it cannot act optimally. They do not say that it is \emph{sufficient}, only that its influence on the optimal action is the same at every horizon. \ref{a:reach} is a property of the environment, not of a policy. 
Setting $T_{\min} \geq 1$ excludes trivial instantaneous decisions at $T=0$.
Cues calling for the same action at every decision step never have to be told apart: we write $[\sigma]$ for the class of $\sigma$ under that equivalence and require separation only between classes. This partition is the $\pi^*$-irrelevance abstraction of \citet{li2006towards}, discussed in \cref{appendix-abstraction}.

\Cref{def:class} is the elementary unit out of which realistic long-horizon partially observable tasks are built (\cref{app:generality}). Its idle phase models any domain where informative observations are sparse and bursty relative to their downstream decisions, from a robot executing an initial instruction after a long traversal to an assistant holding an early user preference across unrelated turns.

\paragraph{Canonical example.}
\begin{wrapfigure}{r}{0.40\textwidth}
  \vspace{-\baselineskip}
  \centering
  \includegraphics[width=0.32\textwidth]{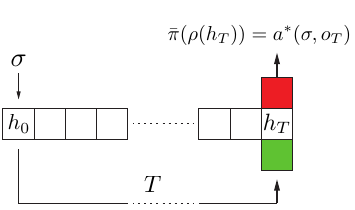}
  \caption{Illustration of the T-maze benchmark.}
  \label{fig:tmaze-env}
  \vspace{-\baselineskip}
\end{wrapfigure}
The T-maze \citep{NIPS2001_a38b1617} is the minimal instance of \cref{def:class} and our working illustration (\cref{fig:tmaze-env}; formal description in \cref{appendix-tmaze}). 
The agent sees the reward side $\sigma \in \{\text{up}, \text{down}\}$ at $t=0$~\ref{a:cue}, crosses a corridor whose observations carry no information about $\sigma$~\ref{a:idle} and whose end is always reachable~\ref{a:reach}, and is signaled the junction by an observation $o_T \in \cO$, where the rewarded action $a^*(\sigma, o_T)$ is independent of the corridor length~\ref{a:dec}. A corridor of length $L$ gives a horizon $T = L-1$, with $T_{\min} = 1$: it can vary without changing the optimal policy.

\subsection{Definition of delay-horizon independence}\label{section-thidef}

\begin{definition}[Delay-horizon independence]
  \label{def:thi}
  Let $\{\mathcal{M}_T\}_{T \geq T_{\min}}$ be a delayed-cue family and let $\pi$ be a policy optimal at $T_{\min}$. Then $\pi$ is \emph{delay-horizon independent} (DHI) if:
    \begin{conditions}[label={\normalfont\condstyle{(\alph*)}},ref={(\alph*)}]       
    \item\label{thi:dec} \emph{Decision independence:} $\pi$ is optimal for all horizons $T \geq T_{\min}$; and
    \item\label{thi:cost} \emph{Cost independence:} The memory footprint of $h_t$, the per-step computation of
the state update and of $\pi$, and the precision at which $h_t$ must be resolved are
uniformly bounded in $T$.
  \end{conditions}
\end{definition}

Decision independence asks for optimality at every
$T \geq T_{\min}$, rather than up to some finite horizon. This requirement \emph{decouples the training horizon from the deployment horizon}: a policy validated on short episodes remains optimal at any longer episode length, without ever being tested at that length. To make the contrast explicit, we also measure how far a policy generalizes when decision independence fails.

\begin{definition}[Effective horizon]
  \label{def:teff}
  The \emph{effective horizon} of a policy $\pi$ is
  \[
    T^\star(\pi) \coloneqq \sup\left\{ T' \geq T_{\min} \mid \pi \text{ is optimal for every } T_{\min} \leq T \leq T' \right\}.
  \]
\end{definition}

By construction, $T^\star(\pi) = \infty$ if and only if $\pi$ satisfies decision independence.
A relaxed property, requiring optimality only up to a finite $T^\star$, is a strictly weaker
claim that does not provide the same guarantee. It is operationally usable only if
$T^\star$ can be predicted in advance and the maximal effective horizon of the environment is known. 
All claims below are \emph{asymptotic} statements about the limit $T \to \infty$.

Cost independence rules out history buffers, whose state grows with the sequence length. Its requirements on memory and computation are immediate for a recurrent policy. The one on precision needs a quantitative form. 




\begin{definition}[Resolvable separation]\label{def:sep}
Let $\cR_\sigma$ denote the set of memory states the agent can be in during the idle phase, once the key observation $\sigma$ has been encoded. The \emph{resolvable separation} is
\[
  \sep \coloneqq \inf\left\{ \|h-h'\| \mid
  h \in \cR_\sigma, h' \in \cR_{\sigma'}, [\sigma] \neq [\sigma'] \right\},
\]
the smallest distance between memory states carrying cues that require different actions.
\end{definition}

By construction, $\sep > 0$ if and only if $\pi$ satisfies the precision requirement of cost independence (\cref{def:thi}\ref{thi:cost}). States that come arbitrarily close yet demand different actions force a precision that grows without bound. The infimum runs over all delays, not over matched ones, because the policy reads its memory without knowing the delay: it must resolve the two sets, not two snapshots taken at the same instant.

Delay-horizon independence is the conjunction of the two: $T^\star(\pi) = \infty$ and $\sep > 0$.

\section{Memory dynamics for delay-horizon independence}
\label{section-lemma}


The dynamics of the memory decide whether a recurrent policy is delay-horizon independent. 
We first hold the idle observation and action constant, which reduces delay-horizon independence from a property quantified over all delays to a property of a single map (\cref{section-cond}). 
We then classify memory dynamics by that map: monostable dynamics preclude delay-horizon independence (\cref{section-mono}), multistable dynamics enable it (\cref{section-multi}). 
We finally let the observation and action vary again, and show that the classification survives (\cref{section-nonconst}). 
\Cref{appendix-scope} delineates the scope of each result, and all proofs are given in \cref{appendix-proofs}.

\subsection{Conditions for delay-horizon independence}
\label{section-cond}

\paragraph{Simplified dynamics during the idle phase.}
We first consider a simplified idle phase ($t = 1, \ldots, T-1$) with constant observation
$o_t=o_c$ and constant action $a_{t-1}=a_c$. \Cref{eq:update} then reduces to a fixed map on
$\cH$,
\begin{equation}
  \idle \colon \cH \to \cH, \qquad \idle(h) \coloneqq \Phi(h, \bar{x}), \qquad \bar{x} \coloneqq (o_c, a_c),
  \label{eq:U}
\end{equation}
which we call the \emph{idle map}, and we write $\idle^n$ for its $n$-fold composition, so that the memory
state at the decision step is $h_{T} = \idle^{T}(h_0(\sigma))$, where
$h_0(\sigma) = \Phi(h_{\mathrm{init}}, (\sigma, a_{\mathrm{init}}))$ denotes the memory state
just after the key observation has been encoded. 
The last update, at $t=T$, is absorbed into the readout introduced below (\cref{appendix-scope}).
We write
\[
  \cR_\sigma \coloneqq \{ \idle^n(h_0(\sigma)) \mid n \geq 0 \}, \qquad
  \cR \coloneqq \bigcup_{\sigma \in \Sigma} \cR_\sigma ,
\]
for the states the memory can occupy after a cue. Holding the idle input constant makes fixed points, attractors and basins properties of a single map.

\paragraph{Readout.} We introduce a \emph{readout function} $\readout \colon \cH \to \cD$,
where $\cD$ is the space of decision-relevant representations. The readout extracts from
$h_t$ the information needed to act, and the policy factorizes as
$\pi = \bar{\pi} \circ \readout$ with $\bar{\pi} \colon \cD \to \cA$. 

\begin{definition}[Invariant and separating readout]
  \label{def:invariant}
  A function $\readout \colon \cH \to \cD$ is \emph{invariant} under the idle map $\idle$ if
  \begin{equation}
    \readout(\idle(h)) = \readout(h), \qquad \text{for all } h \in \cR.
    \label{eq:invariance}
  \end{equation}
    It is \emph{separating} if it distinguishes cues in different classes,
  \[
    [\sigma] \neq [\sigma']
    \implies
    \readout(h_0(\sigma)) \neq \readout(h_0(\sigma')).
  \]
  Invariance is imposed on $\cR$ only, and is not required for every $h \in \cH$.
\end{definition}



\begin{figure*}[t]
  \centering
    \centering
    \includegraphics[width=0.73\linewidth]{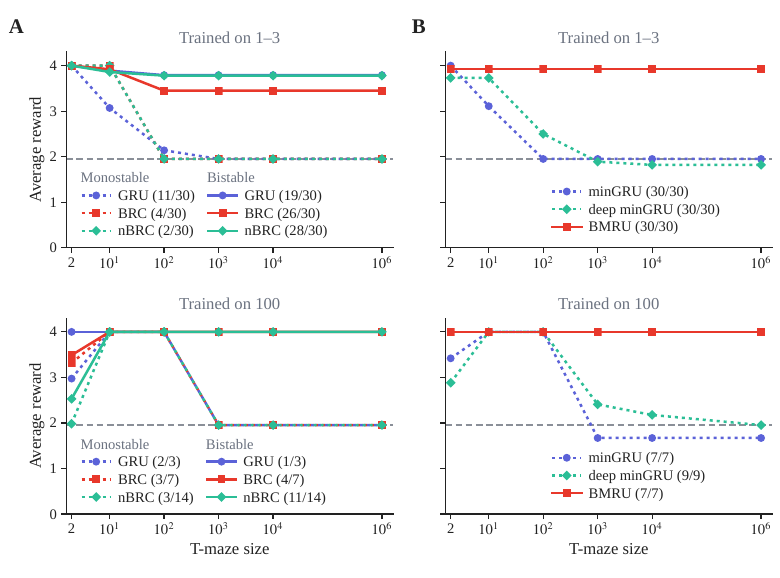}
    \caption{\textbf{Generalization capabilities on the T-maze.} 
    \textbf{A.}~Average reward obtained by each nonlinear RNN as a function of T-maze size, for models trained at $T \in [1,3]$ (top) and at $T = 100$ (bottom). 
    \textbf{B.}~The same for
    the parallelizable RNNs. In both panels, monostable models are dashed and bistable
    models solid, and  the gray dashed line marks
    random choice.}
    \label{fig:tmaze}
\end{figure*}

\begin{lemma}[Readout characterization of delay-horizon independence]
  \label{lem:readout}
  Let $\{\mathcal{M}_T\}_{T \geq T_{\min}}$ be a delayed-cue family and let $\pi$ be a recurrent policy, optimal at $T_{\min}$. Then $\pi$ is delay-horizon independent if and only if it admits a factorization $\pi = \bar{\pi} \circ \readout$ whose readout is uniformly continuous, invariant under the idle map $\idle$, and separating.
\end{lemma}



The content of \cref{lem:readout} is the change of quantifier. Optimality at every horizon cannot be established by testing, since the delays are unbounded. The lemma replaces it by a one-step invariance on $\cR$, which is a property of a single map. The next two subsections exploit this, since monostability and multistability are properties of $\idle$ and not of its iterates.

In what follows, we use \cref{lem:readout} to study which memory dynamics support DHI, and verify theoretical statements experimentally on the T-maze environment: we train recurrent policies at short
($T$ drawn in $1$--$3$) and medium ($T = 100$) horizons, then evaluate them, without retraining, at horizons up to $10^6$ and extract the effective horizon $T^\star$ when it exists. We study three nonlinear cells, the GRU \citep{cho2014learningphraserepresentationsusing}, bistable recurrent cell (BRC) and its neuromodulated variant (nBRC) \citep{vecoven2021bioinspiredbistablerecurrentcell}, and two parallelizable ones, minGRU \citep{feng2024weregrudminimalgatedrecurrent} and BMRU \citep{degeeter2026parallelizable} (+ Mamba \citep{gu2024mambalineartimesequencemodeling} for the $T^{\star}$ analysis), with 30 seeds each. Their equations, training procedure, hyperparameters and network sizes are given in \cref{appendix-RNNdyn,appendix-hyperparameters}. The average reward achieved by each model at all tested lengths is shown in \cref{fig:tmaze} and discussed below.

\subsection{Monostability precludes delay-horizon independence}\label{section-mono}

\begin{definition}[Monostable map]
  \label{def:mono}
  The idle map $\idle \colon \cH \to \cH$ is \emph{monostable} (or globally attracting) if it has a
  unique globally attracting fixed point $\hstar \in \cH$:
  \[
    \idle^n(h) \to \hstar \quad \text{as }n \to \infty,
    \quad \forall\, h \in \cH.
  \]
\end{definition}

\begin{theorem}[Monostability precludes delay-horizon independence]\label{thm:mono}
Let the idle map $\idle$ be monostable with attractor $h^\star$, and suppose at least two cues
require different actions at the decision step. Then $\sep = 0$.
Consequently no recurrent policy whose memory is updated by $\idle$ satisfies
cost independence, and none achieves DHI.
\end{theorem}

The proof (\cref{appendix-proofs}) never uses the readout: monostability defeats horizon independence
through the memory itself. Requiring in addition that the policy act at a
fixed precision recovers the stronger statement that the decision is
lost as well.

\begin{corollary}[Bounded effective horizon]\label{cor:tstar}
Let the idle map $\idle$ be monostable, and let the policy act at a fixed precision: there is an $\eta>0$ such that $\pi(h)=\pi(h')$ whenever $h,h'\in\cR$ and $\|h-h'\|<\eta$. Then $T^\star(\pi)<\infty$, for every choice of parameters and training procedure.
\end{corollary}

\Cref{thm:mono,cor:tstar} are verified experimentally on the T-maze  (\cref{fig:tmaze}A). Extracting the monostable nonlinear models from our trained population (using the variability among attractors metric, \cref{appendix-vaa}), every one of them has a finite effective horizon: it is optimal between $T_{\min}$ and $T^\star$ and decays to random
behavior beyond it, whether trained at the short or at the medium horizon. 

For an affine recurrence $h_t = A(x_t)h_{t-1} + B(x_t) x_t$, monostability can be verified in closed form and without training: under a constant idle input $\bar{x}$, the idle map is affine, with a unique and globally attracting fixed point whenever the spectral radius satisfies $\rho(A(\bar{x}))<1$ (\cref{appendix-RNNdyn}). This class contains SSMs such as Mamba \citep{gu2024mambalineartimesequencemodeling} and gated linear RNNs such as minGRU \citep{feng2024weregrudminimalgatedrecurrent}, all of which are therefore excluded from DHI at every contracting parameter value. Selectivity and input gating can only change the \emph{rate} of decay: they lengthen $T^\star$ without making it infinite. The noncontracting boundary, where $A(\bar{x})$ has an eigenvalue on the unit circle, is treated in \cref{rem:boundary}.

We confirm this experimentally on the T-maze using minGRU (\cref{fig:tmaze}B): as for the nonlinear monostable models, the effective horizon of minGRU is bounded, and increased capacity does not repair it: a deep minGRU (four layers of 128 units, against one layer of five) lengthens $T^\star$, but performance still decays to random, only further along the horizon axis, as \cref{thm:mono} predicts. 

\paragraph{$T^\star$ is not a design parameter.} 
The horizon of a monostable policy is an emergent quantity, which can be established only by testing at the deployment horizon. A finite $T^\star$ would be good enough if tuning could make it large.
 It cannot,
because it cannot be predicted. In a hyperparameter search across $250$ configurations per architecture (\cref{fig:tstar}A), minGRU, deep minGRU and Mamba solve the task at the training horizon in $80$, $79$ and $89$ configurations respectively. Across the search, $T^\star$ spans more than five orders of magnitude, up to the evaluation ceiling of $10^6$, which $14$ deep minGRU and two Mamba configurations reach. A long horizon can therefore be selected for, but retraining the selected configuration on different seeds leads to highly inconsistent horizons (\cref{fig:tstar}B): only the two shortest targets are reproduced within a factor of two, and for finite targets the median over $30$ seeds
falls short of the value that motivated the selection by up to $745\times$. Reaching the ceiling in the search is no exception: the configurations selected there fall back to a median $T^\star$ of $5{,}151$ (deep minGRU) and $2.5$ (Mamba), and at least half of these Mamba seeds no longer solve the task at the training horizon. 


\begin{figure*}[t]
    \centering
     \includegraphics[width=\textwidth]{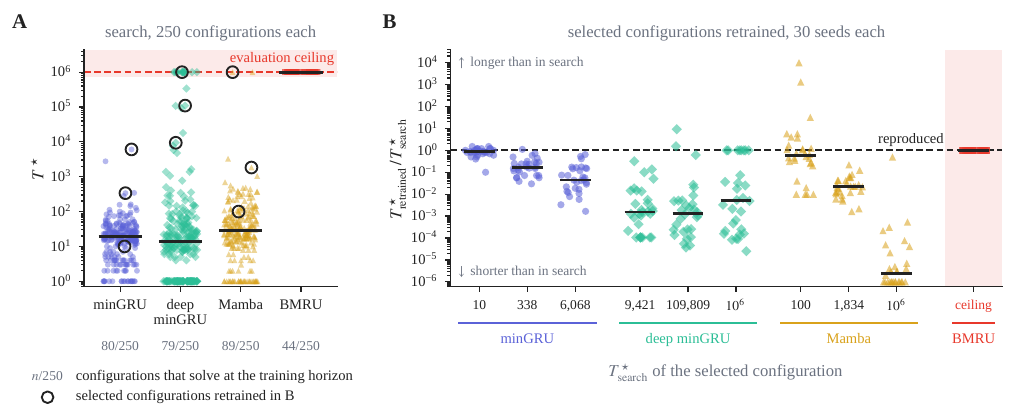}
    \caption{\textbf{$T^\star$ is unpredictable for monostable memory.} \textbf{A.} $T^\star$ of every configuration in a search of 250 configurations per architecture on the T-maze (for BMRU, the 44 configurations that solve the task). Fractions below the labels count the configurations that solve the task at the training horizon; black bars are medians; the shaded band is the evaluation ceiling ($T = 10^6$). Circled configurations are retrained on 30 seeds in B. \textbf{B.} $T^\star$ after retraining, as a ratio of the $T^\star$ that configuration reached in the search (dashed line: target reproduced).}
\label{fig:tstar}
\end{figure*}

\subsection{Multistability enables DHI}\label{section-multi}
\begin{definition}[Multistable map]
  \label{def:multi}
  The idle map $\idle \colon \cH \to \cH$ is \emph{multistable} if it has $K \geq 2$ disjoint
  attractors $A_1,\dots,A_K \subset \cH$, compact invariant sets that each attract a neighborhood of itself in $\cH$, with basins of attraction
  $B_k = \{h \in \cH \mid \dist(\idle^n(h), A_k) \to 0 \text{ as } n \to \infty\}$. 
  Fixed-point attractors ($A_k = \{h_k^\star\}$) are the simplest
  case, but limit cycles and more general compact invariant sets are included.
\end{definition}

\begin{theorem}[Multistability enables delay-horizon independence]
  \label{thm:multi}
    Let $\idle$ be multistable. Suppose:
  \begin{enumerate}[label=(\roman*),nosep,leftmargin=*]
    \item\label{itm:ab1} For each $\sigma$, the encoding places $h_0(\sigma)$ in a
          specific basin $B_{k(\sigma)}$, \textit{i.e.},  the acquisition phase
          is well-posed, and $[\sigma] \neq [\sigma']$ implies
          $k(\sigma) \neq k(\sigma')$.
    \item\label{itm:ab2} There exists a readout $\readout \colon \cH \to \cD$ that
          \begin{itemize}
            \item is invariant within each basin:
                  $\readout(\idle(h)) = \readout(h)$ for all $h \in B_k$,
                  $k = 1, 2, \dots,K$, and
            \item separates the basins: $\readout(h) \neq \readout(h')$ with $h\in B_{i}$,
                  $h'\in B_{j}$ and $i\neq j$.
          \end{itemize}
  \end{enumerate}
    If $\pi = \bar{\pi} \circ \readout$ is optimal at $T_{\min}$, then it is delay-horizon independent.
\end{theorem}

Assumption (i) places the memory state in the basin the cue selected, and it stays there because basins of attraction are forward invariant; (ii) makes the readout constant along that orbit. 
Cost independence is not automatic and is checked in the proof: distinct attractors are disjoint compact sets, hence at positive distance, so the orbits seeded by two cues in different basins stay apart. \Cref{appendix-proofs} gives the details.


\Cref{thm:multi} is likewise verified experimentally on the T-maze environment. The multistable models among the trained nonlinear RNNs are horizon-independent over
the whole evaluated range, from $T = 2$ to $T = 10^6$, and this holds across RNN types, for
GRU, BRC and nBRC (\cref{fig:tmaze}A). Multistability is reached by a fraction of the
trained models rather than by all of them, so realizing DHI in the cells in widest use is a question of training procedures that foster multistability, not of
architectural capacity.

We further confirm that the mechanism is multistability itself, not the transient
dynamics that accompany it, using the BMRU \citep{degeeter2026parallelizable}: a
parallelizable RNN that is multistable by construction and converges to equilibrium in
a single step, and therefore carries persistent memory with no transient component. It
is optimal at every maze length and in every training condition (\cref{fig:tmaze}B).
Contrasting it with minGRU, which has transient dynamics but a single attractor, supports
the causal role of multistability: the model with transients and no multistability fails
at a finite $T^\star$, the model with multistability and no transients does not. Under the
search of in \cref{section-mono}, BMRU solves
the task in $44$ of $250$ configurations, all $44$ sit at the evaluation ceiling, and every
retrained seed reproduces it (\cref{fig:tstar}). The unpredictability documented earlier is
thus a property of monostable memory, not of the search.


\subsection{Beyond the constant idle phase}
\label{section-nonconst}

In practice, the agent keeps acting and observing during the idle phase: observations
fluctuate and the agent takes exploratory or habitual actions, or achieves intermediate
goals. As long as these perturbations never push the memory state across the boundary
between two basins of attraction, the readout $\readout$ remains informative and delay-horizon independence is preserved.

\begin{proposition}[Non-constant idle phase]
  \label{prop:nonconst}
  Let the input $x_t$ vary with time.
  Suppose that assumptions (i) and (ii) of \cref{thm:multi} hold with the basins $B_1, \dots, B_K$ replaced by regions that are pairwise at positive distance and forward invariant under every $\Phi_x$, and with the invariance of $\readout$ required under every $\Phi_x$.
  Then the conclusions of \cref{thm:multi} hold with
  $\idle^T$ replaced by the composition $\Phi_{x_T} \circ \cdots \circ \Phi_{x_1}$.
\end{proposition}


\subsection{Multistability is the only mechanism}\label{section-only}

\Cref{thm:mono,thm:multi} do not by themselves make multistability necessary, because a map can be neither monostable nor multistable. An idle map that is the identity on a subspace, a perfect integrator or a line attractor, has no attractor in the sense of \cref{def:multi}, yet its orbits stay apart and it admits an invariant separating readout. What excludes such maps is that an idle phase is never exactly one map: observations fluctuate, the agent acts, and arithmetic rounds. Requiring delay-horizon independence to survive arbitrarily small perturbations forces genuine attractors.

\begin{theorem}[Robust delay-horizon independence requires multistability]\label{thm:robust}
Let $\cH \subset \R^d$ be compact, let $\idle$ be continuous, let $\eta>0$, and let the idle phase be realized at each step by some map drawn from $\cW_\eta \coloneqq \{ \tilde{\Phi}_{\bar{x}} \colon \cH \to \cH \mid \sup_{h\in\cH}\| \tilde{\Phi}_{\bar{x}}(h)-\idle(h) \| \leq \eta \}$.
Suppose a policy $\pi = \bar\pi \circ \readout$ is delay-horizon independent for \emph{every} realization of the idle phase in $\cW_\eta$, with a readout $\readout$ that is uniformly continuous, invariant along every realization and separating, and that at least two cues require different actions.
Then $\idle$ has at least two disjoint attractors: $\idle$ is multistable in the sense of \cref{def:multi}.
\end{theorem}

The proof, in \cref{appendix-proofs}, shows that the states reachable from each cue form a compact trapping region, $\idle(\cR_\sigma)$ lying in its relative interior, and that a trapping region contains an attractor.

\begin{remark}[The boundary cases]\label{rem:boundary}
    \Cref{thm:robust} excludes the maps that \cref{thm:mono} does not reach. 
    A contraction has no two disjoint trapping regions. On a connected $\cH$, the identity has only one, $\cH$ itself, since a closed set with $(\cR_\sigma \oplus \bar B_\eta) \cap \cH \subseteq \cR_\sigma$ is also open in $\cH$. This disposes of the perfect integrator, of an affine recurrence whose $A(\bar{x})$ has an eigenvalue on the unit circle, and of unnormalized linear attention, whose recurrent form $S_t = S_{t-1} + \phi(k_t)v_t^{\Tr}$ is a pure accumulator. 
    The identity survives only on a discrete set of states, where each state is isolated and hence an attractor, as in the BMRU, whose units hold $\pm\alpha_i$ unchanged during the idle phase. 
    Only attracting dynamics absorb the perturbation, and two decisions require two attractors.
\end{remark}

Multistability is therefore not one design option among several: on a memory whose size does not grow with the delay, it is the \emph{only} mechanism that supports delay-horizon independence robustly. The apparent alternative, letting the memory grow, does not escape the conclusion: it moves the failure from one requirement of \cref{def:thi} to the other. \Cref{appendix-scope} states the scope of each result explicitly and works through the cases, attention-based policies and external memory in particular.



\section{Delay-horizon independence on standard POMDP benchmarks}
\label{section-experiments}

Both benchmarks reproduce the two patterns established on the T-maze: monostable memory has a finite effective horizon, and that horizon is a property of the individual trained model, not of the architecture. We evaluate the parallelizable models on two partially observable benchmarks adapted from MiniGrid, each chosen to stress a different aspect of \cref{def:class}: MiniGrid-Memory extends the T-maze with richer dynamics and egocentric observations, and MiniGrid-GoToDoor extends it to 2D navigation with four possible decisions (\cref{appendix-memory,appendix-gotodoor}). In both, the cue-to-decision delay grows at least linearly with the grid size $N$, so the horizon axis of \cref{section-lemma} is the size axis of \cref{fig:benchmarks}.
Models were trained only at short horizons $N \in \{7, 9, 11\}$ and evaluated up to $N = 10^{6}$, without retraining (\cref{fig:benchmarks}). 

\begin{figure*}[t]
    \centering
     \includegraphics[width=0.9\textwidth]{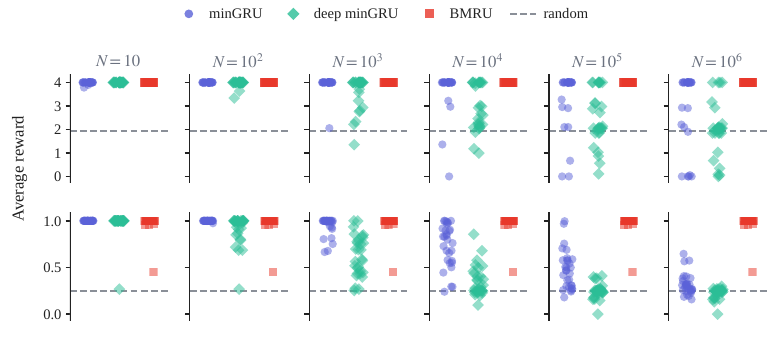}
    \caption{\textbf{Horizon independence on standard POMDP benchmarks.} Average reward
    as a function of the grid size $N$, for models trained only at short horizons and
    evaluated up to $N = 10^{6}$ without retraining. Top row, MiniGrid-Memory; bottom row,
    MiniGrid-GoToDoor; the two environments have different reward scales, and the
    gray dashed line marks random choice in each. Every marker is one seed: minGRU (blue
    circles), deep minGRU (green diamonds) and BMRU (red squares). Both monostable models are
    optimal at and just beyond the training sizes and then decay to random behavior, deep
    minGRU only slightly later than minGRU, whereas the BMRU distribution is unchanged from
    the training sizes up to $N = 10^{6}$.}
    \label{fig:benchmarks}
\end{figure*}

Up to and slightly beyond the training sizes, the three architectures are optimal and indistinguishable, with nearly all seeds on top of each other. Beyond them, the two monostable populations \emph{spread}. At every intermediate grid size, part of the seeds still acts
on the cue while the rest have already lost it, and the spread widens with $N$. On MiniGrid-Memory, the first seeds leave the optimum at $N = 10^{2}$--$10^{3}$, the population straddles the random-behavior line from
$N = 10^{4}$ on, and a fraction of seeds is still optimal at $N = 10^{6}$, three orders of magnitude past the first failures of the same architecture trained identically. MiniGrid-GoToDoor shows the same picture with an earlier onset. Adding capacity does not change the nature of the failure: the deep minGRU population overlaps with the shallow one throughout, with no systematic gain in the size at which it fails.


The BMRU distribution over seeds at $N = 10^{6}$ is the one it has at the training sizes, which is the empirical signature of \cref{thm:multi}. One detail of \cref{fig:benchmarks} makes the point sharper: on MiniGrid-GoToDoor a single BMRU seed sits at about half the optimal reward, at every grid size \emph{including the training ones}. That seed is a
suboptimal policy, not a policy that forgets. Horizon
independence constrains the dependence of the score on the horizon, not its value.

The idle phase is non-constant in both environments: the egocentric observations change with the agent own motion in MiniGrid-Memory, and the agent follows a two-dimensional trajectory of its choosing in MiniGrid-GoToDoor. These results are therefore consistent with \cref{prop:nonconst}: the perturbations of the idle phase do not push the memory state out of the region in which the cue placed it, and the cue survives an unbounded number of them.

\section{Discussion}\label{section-discussion-limitations}
In this paper, we introduced the concept of delay-horizon independence in RL, showed how it decouples the training horizon from the deployment horizon, and linked the ability of RNN variants to reach DHI with their dynamical properties. Through a mix of theoretical developments and experimental validations, we have shown that, on a recurrent state whose size does not grow with the delay, multistability is the only mechanism supporting DHI once the idle phase is subject to arbitrarily small perturbations (\cref{thm:robust}), and that it is absent by construction from most state-of-the-art parallelizable RNN variants.

\textbf{On the importance of an asymptotic guarantee.} Our claims are asymptotic. \Cref{thm:mono} says that a monostable model cannot be optimal at \emph{every} horizon; it does not say that such a model fails at any given horizon, which is confirmed by our experiments. What we claim is that the boundary of that range, the effective horizon $T^\star$, is finite and, empirically, unpredictable: it varies by orders of magnitude across seeds within a fixed architecture and training setup, so it cannot be relied upon without testing at the deployment horizon, which reintroduces the very cost DHI is meant to remove.

\textbf{Limitations.} \textit{On the sufficiency of multistability.} While our theoretical results establish multistability as sufficient (\cref{thm:multi}) and, under perturbation, necessary (\cref{thm:robust}) for the \emph{existence} of delay-horizon independent optimal policies, reaching these policies during training is not guaranteed. Developing generalization-aware training strategies that explicitly foster multistability \citep{lambrechts2023warmingrecurrentneuralnetworks}, which would also be the route to transferring these results to widely deployed cells such as the GRU and the LSTM, is an essential next step.

\textit{The switch toward parallelizable architectures.} As the field trends toward parallelizable RNN variants to enable large-scale computation, a critical trade-off has emerged. We showed that current state-of-the-art models such as SSMs and gated linear RNNs are inherently monostable, making them incompatible with DHI. As a solution, we propose a shift toward hybrid architectures that integrate the efficiency of parallelizable transient dynamics with the persistent memory of multistability. Such hybrid designs could lead to scalable, efficient models that are delay-horizon independent.


\newpage

\subsection*{AI use statement}
In this work, we partially used generative AI tools to formulate mathematical claims, assist in the writing of proofs, design or provide feedback on research methodology or experiments, implement methods, and assist with translation.

We did not use generative AI tools to interpret results, propose or refine hypotheses, or provide critical ingredients to prove mathematical claims. Generating synthetic datasets, developing theoretical models or conceptual frameworks, cleaning or reformatting datasets, and qualitative or thematic data analysis are not applicable to this work.

Additionally, we partially used generative AI tools to identify relevant literature, create or edit software code, and edit the research paper to improve readability, including figures. 

We have reviewed all AI-assisted work. Specifically, all mathematical claims, proofs, and methodological designs were manually checked for rigor and technical accuracy by the authors; AI-generated and edited code was verified and unit-tested; translated text was checked for semantic accuracy; and all referenced literature and information retrieved via AI were independently verified against original source publications. We take responsibility for the final content of this work, including text, claims or artifacts produced with the aid of generative AI.




\section*{Acknowledgements}
This work has been the subject of a patent application (Number: EP26151077). This present research was supported by the Belgian Government through the Federal Public Service Policy and Support. It also benefited from computational resources made available on Lucia, the Tier-1 supercomputer of the Walloon Region, infrastructure funded by the Walloon Region under the grant agreement n°1910247. Florent De Geeter gratefully acknowledges the financial support of the Walloon Region for Grant No. 2010235 – ARIAC by DW4AI.

\bibliography{refs}
\bibliographystyle{iclr2027_conference}

\clearpage
\appendix

\section*{Contents of this appendix}
\begin{itemize}
\item \Cref{appendix-related-work}: \nameref{appendix-related-work}
\item \Cref{appendix-pomdp}: \nameref{appendix-pomdp}
\item \Cref{app:generality}: \nameref{app:generality}
\item \Cref{appendix-DHI}: \nameref{appendix-DHI}
\item \Cref{appendix-RNNdyn}: \nameref{appendix-RNNdyn}
\item \Cref{appendix-tmaze}: \nameref{appendix-tmaze}
\item \Cref{appendix-memory}: \nameref{appendix-memory}
\item \Cref{appendix-gotodoor}: \nameref{appendix-gotodoor}
\item \Cref{appendix-hyperparameters}: \nameref{appendix-hyperparameters}
\item \Cref{appendix-sweep}: \nameref{appendix-sweep}
\item \Cref{appendix-teff}: \nameref{appendix-teff}
\item \Cref{appendix-vaa}: \nameref{appendix-vaa}
\item \Cref{appendix-supFig}: \nameref{appendix-supFig}
\item \Cref{appendix-RNNeq}: \nameref{appendix-RNNeq}
\end{itemize}

\section{Related Work}\label{appendix-related-work}
Since our work sits at the intersection of several active research areas, we briefly survey three relevant topics in this section: generalization in RL, memory for POMDPs, and multistability in RNNs.

\paragraph{Generalization in reinforcement learning.}
Generalization is a long-standing challenge in reinforcement learning: policies learned in one setting often fail to transfer even to superficially similar environments \citep{korkmaz2024surveyanalyzinggeneralizationdeep}. Early work highlighted that deep RL agents can overfit dramatically to the specifics of their training environments, including random seeds and initial state distributions, even when the underlying task remains the same \citep{zhang2018dissection}. This prompted the development of dedicated generalization benchmarks such as CoinRun \citep{cobbe2019quantifying} and Procgen \citep{cobbe2020leveraging}, which expose policies to procedurally generated environments at test time, and systematic evaluations of what dimensions of variation agents can and cannot handle.

A comprehensive survey \citep{kirk2023survey} categorizes generalization along several axes: across environment parameters (reward, dynamics, layout), across tasks, and across levels of difficulty, each presenting distinct failure modes. One practical strategy to improve generalization is curriculum learning, in which the agent is progressively exposed to harder instances of the task, reducing the exploration burden and improving sample efficiency \citep{bengio2009curriculum, narvekar2020curriculum}. A closely related strategy is domain randomization, which widens the training distribution to encourage policies that are robust to unseen variations \citep{tobin2017domain}.

Generalization in fully observable settings is already difficult \citep{ghosh2021generalizationrldifficultepistemic} and partial observability amplifies the difficulty because the agent cannot distinguish states that appear identical but require different actions.

Our work focuses on a specific, formally defined axis of generalization: the temporal horizon. \citet{myers2025horizongeneralizationreinforcementlearning} introduced \textit{horizon generalization} in goal-conditioned MDPs and connected it to planning invariance, whereby an agent pursues the same policy regardless of whether it is given a distant goal or an intermediate waypoint. We study the partially observable setting, where memory is the central bottleneck, and where the property of interest is stronger: not that behavior transfers from short to long horizons, but that neither the decision nor its cost depends on the horizon at all. We provide a formal necessary and sufficient condition linking that independence to the dynamical properties of the agent's recurrent memory. Unlike curriculum-based or domain-randomization approaches, which seek to generalize by widening the training distribution, delay-horizon independence allows an agent trained exclusively on short horizons to remain optimal at arbitrarily longer ones, with no requirement to expose it to long-horizon episodes during training.

\paragraph{Memory and recurrent architectures for POMDPs.}
Partial observability is ubiquitous in real-world sequential decision making. A principled solution is to maintain a belief state, a probability distribution over the true environment state conditioned on the observation history \citep{kaelbling1998planning}. Exact belief-state computation is generally intractable, and in deep RL it is typically replaced by a learned approximate summary encoded in the hidden state of a recurrent neural network. \citet{hausknecht2015deep} proposed the deep recurrent Q-network (DRQN), which replaces one of the feedforward components of DQN \citep{mnih2015human} with an LSTM layer, showing that a recurrent hidden state can substitute for full observability in Atari games. Equipping actor-critic algorithms such as A3C \citep{mnih2016asynchronous} and PPO \citep{schulman2017proximalpolicyoptimizationalgorithms} with recurrent encoders has since become one of the standard approaches to handling partial observability.

Beyond standard RNNs, various approaches have been proposed to enhance the memory capacity accessible to the agent. Memory-augmented networks such as the neural turing machine \citep{graves2014neural} and the differentiable neural computer \citep{graves2016hybrid} decouple capacity from the hidden-state dimension by equipping networks with an addressable external memory, enabling longer retention at the cost of additional architectural complexity. Transformer-based architectures provide a complementary strategy via self-attention over a growing context window. The gated transformer-XL (GTrXL) \citep{parisotto2020stabilizing} adapts this idea for RL by incorporating gating to stabilize training, achieving strong results on memory-intensive benchmarks. However, self-attention is quadratic in context length, which becomes prohibitive for very long horizons, and attention over a finite window cannot retain the cue beyond that window. Attention therefore illustrates the trade at the heart of \cref{def:thi}: a growing context keeps the decision correct at a cost that grows with the horizon, and a fixed context keeps the cost bounded at the price of a horizon ceiling.

Structured world models offer a further alternative, learning a latent dynamics model of the environment from which a compact belief state can be inferred. Dreamer and its successors \citep{hafner2019dreamerv1,hafner2020dreamerv2,hafner2025dreamerv3} use a recurrent state space model to separate deterministic from stochastic latent dynamics, enabling policy learning and value estimation through imagined rollouts. VariBAD \citep{zintgraf2021varibad} takes a Bayesian perspective, meta-learning a posterior over task parameters that is updated online to approximate the belief state. These approaches enrich memory through architectural or training innovations, but do not address the dynamical question of whether the underlying recurrent cell can sustain the persistent, history-dependent representations that delay-horizon independence demands. Our analysis shows that this is fundamentally a question of multistability in the cell dynamics, a property that is absent in linear and input-gated architectures regardless of how they are embedded in a larger system.

\paragraph{Multistability in RNNs.}
Monostability and fading memory are closely related concepts~\citep{boyd1985fadingmemory}: a system with a single stable equilibrium forces all trajectories to decay exponentially toward that unique attractor, inevitably erasing information about past inputs, which is the key characteristic of fading memory~\citep{boyd1985fadingmemory, pascanu2013difficultytrainingrecurrentneural, bainier2026state}. In contrast, persistent memory requires multistability, where multiple distinct stable states can coexist, enabling the network to maintain and switch between stored values without asymptotic decay. Despite this fundamental link, \textit{multistability remains largely neglected in the machine learning literature on recurrent architectures} \citep{liu2023stability}; notable exceptions include bio-inspired bistable cells~\citep{vecoven2021bioinspiredbistablerecurrentcell,lambrechts2023warmingrecurrentneuralnetworks} and the MRU line of works~\citep{degeeter2026parallelizable, brandoit2026improving}, both of which explicitly harness multistable dynamics to overcome the inherent forgetting of monostable linear systems. 

\newpage
\section{Reinforcement Learning in POMDPs}
\label{appendix-pomdp}

Formally, a POMDP $\mathcal{M}$ is defined as an 8-tuple
$$\mathcal{M} = (\mathcal{S}, \mathcal{A}, \mathcal{O}, p_0, P, \Omega, R, \gamma),$$ where $\mathcal{S}$ is the state space, $\mathcal{A}$ is the action space, and $\mathcal{O}$ is the observation space. The initial state distribution
$$p_0(s_0) \colon \mathcal{S} \to \left[0, 1\right]$$
gives the probability of $s_0 \in \mathcal{S}$ being the initial state of the POMDP. The Markovian system dynamics are defined by the transition probability distribution
$$P(s_{t+1} \mid s_t, a_t) \colon \mathcal{S} \times \mathcal{A} \times \mathcal{S} \to \left[0, 1\right],$$ which gives the probability that taking action $a_t \in \mathcal{A}$ in state $s_t \in \mathcal{S}$ leads to state $s_{t+1} \in \mathcal{S}$. The reward function
$$R(s_t, a_t, s_{t+1}) \colon \mathcal{S} \times \mathcal{A} \times \mathcal{S} \to \mathbb{R}$$
gives the immediate reward for transitioning from state $s_t \in \mathcal{S}$ to state $s_{t+1} \in \mathcal{S}$ through action $a_t \in \mathcal{A}$. The observation distribution $$\Omega(o_t \mid s_t) \colon \mathcal{O} \times \mathcal{S} \to \left[0, 1\right]$$ gives the probability of getting observation $o_t \in \mathcal{O}$ in state $s_t \in \mathcal{S}$. Finally, $\gamma \in \left(0, 1\right]$ is a discount factor that scales the impact of future rewards.

The transition law is denoted $P$ and the observation law $\Omega$, rather than $T$ and $O$ as is also common, to avoid any collision with the horizon $T$ and the observation space $\mathcal{O}$, both of which are central to this work.

When dealing with a POMDP in reinforcement learning, the agent cannot access the state $s_t \in \mathcal{S}$ directly and must rely on the history $\tau_t = (o_0, a_0, \ldots, o_t) \in (\mathcal{O}, \mathcal{A})^t \times \mathcal{O} \coloneqq \mathcal{T}_t$ since observations in a POMDP do not provide sufficient information as $P(o_{t+1} \mid \tau_t, a_t) \neq P(o_{t+1} \mid o_t, a_t)$.

\begin{definition}[Optimal policy]
  \label{def:optimal}
  A policy $\pi$ is \emph{optimal} for $\mathcal{M}$ if it maximizes the expected discounted return $\mathbb{E}\left[\sum_{t \geq 0} \gamma^{t} R(s_t, a_t, s_{t+1})\right]$ over all history-dependent policies.
\end{definition}

Due to the unbounded growth of the history over time, a common approach is to encode the history into a fixed-size latent state using RNNs. This latent representation can then be used to learn a history-dependent policy with standard reinforcement learning algorithms.

\newpage
\section{Generality of the Delayed-Cue POMDP}
\label{app:generality}

Definition~\ref{def:class} isolates a single instance of the cue/idle/decision pattern for analytical tractability. This appendix argues that the pattern is not a narrow special case but the elementary unit out of which realistic long-horizon partially observable tasks are built: such tasks rarely present as one clean instance of Definition~\ref{def:class}, but they contain it as a recurring sub-structure.
We give four examples spanning distinct domains (Section~\ref{app:generality-examples}),
then discuss how single-cue results compose when several such segments interact
(Section~\ref{app:generality-composition}).

\subsection{Delayed-Cue across domains}
\label{app:generality-examples}

\textbf{Manipulation under transient occlusion.} A manipulator observes an object identity or target pose only during an initial approach; the object then becomes occluded (by the gripper, by another object, by a viewpoint change) for the remainder of the reach, and the manipulation action must be correct once contact is made. The occluded interval is uninformative about identity \textbf{(A2)} and its length is set by the approach trajectory, not by the object \textbf{(A4)}; the action at contact is the decision step \textbf{(A3)}, and the observation at the start is the cue \textbf{(A1)}. This is the T-maze corridor played out as a reach instead of a walk, with the horizon set by geometry rather than by design.

\textbf{Instruction-conditioned navigation.} An agent is told its destination once, at the start of an episode, then must traverse a long, task-irrelevant environment (corridors, junctions, other rooms) before the destination becomes actionable at a door or fork. The intervening steps are exploratory and carry no information about the destination; only the terminal choice depends on it, and the number of intervening steps is set by the building layout, not by the task designer.

\textbf{Assistive agents with a stated early preference.} A user states a preference once, early in an interaction (a dietary restriction, a seating
preference), and the agent then handles several unrelated requests before that preference becomes relevant to a downstream choice (an order, a booking). The intervening turns carry information about the unrelated requests but nothing about the stated preference, and their number is set by the conversation, not fixed in advance.

\textbf{Delayed alarms in monitoring and control.} A supervisory controller receives a diagnostic reading once (an anomaly class flagged by a sensor) and must
act on it correctly whenever the corresponding fault later manifests, after an arbitrary number of nominal control steps that reveal nothing about which anomaly
class was flagged. The delay between flag and fault is not known at design time.

In each case the cue is cheap to observe but costly to forget, and the interval
separating it from the decision is not a fixed constant: it is set by an approach
trajectory, a building layout, a conversation length, or a time-to-failure that the agent does not control. This is exactly why Definition~\ref{def:thi} quantifies over every horizon $T \geq T_{\min}$ rather than fixing one: none of the examples above come with a guaranteed bound on how long the idle phase will last, so a policy that is only validated at one horizon offers no guarantee at deployment.

\subsection{Composing single-cue results}
\label{app:generality-composition}

Realistic tasks rarely stop at one cue. Two extensions matter: cues that arrive and
are resolved sequentially, and cues that must be held concurrently.

\textbf{Sequential cues.} Suppose the idle phase does not end with a single
decision but with a fresh cue $\sigma_2$ at time $\tau$, followed by a second idle
phase and a decision at $T > \tau$ that depends on $\sigma_2$ (possibly together with
$\sigma_1$). This is exactly the setting covered by
\cref{prop:nonconst}: re-anchoring $t=0$ at $\tau$ and treating the readout of $\sigma_1$ as part of the (now time-varying) input to the second segment
recovers the same forward-invariance argument, applied once per segment. A
multistable memory that resolves the first decision and then supports a (possibly
different) multistable regime for the second segment inherits delay-horizon
independence on each segment separately, and hence on their concatenation, provided
the transition at $\tau$ does not itself cross a basin boundary before $\sigma_2$ is
encoded: the same non-crossing condition as assumption~\emph{(i)} of \cref{prop:nonconst}, applied at every cue arrival rather than only at $t=0$.

\textbf{Concurrent cues.} The harder case is a decision at $T$ that depends
jointly on two cues $\sigma_1$ (at $t=0$) and $\sigma_2$ (at $t=\tau$) that must both
be retained simultaneously. \Cref{thm:multi} already covers an arbitrary finite number of attractors $A_1, \ldots, A_K$ for a single cue drawn from a set of size $K$; concurrency requires instead that the attractor landscape represent the Cartesian product of the two cue alphabets, i.e., at least $|\Sigma_1| \times |\Sigma_2|$ distinct basins rather than $|\Sigma_1| + |\Sigma_2|$. This is the same combinatorial structure as the register construction of \cref{appendix-scope} ($2^b$ attractors for a $b$-bit register): a memory holding $k$ simultaneously outstanding cues needs an attractor count that grows multiplicatively in $k$.

\textbf{The failure of monostable maps generalizes.} \Cref{thm:mono} is a negative result about monostable memory, which gets stronger under composition: if a monostable idle map cannot retain even one cue across one uninformative interval, it cannot retain several, whether they arrive sequentially or concurrently. The single-cue analysis therefore identifies the failure mode that
recurs in every composed setting above.

\newpage
\section{Additional theoretical material}
\label{appendix-DHI}

This appendix collects the proofs of all results stated in the main text (\cref{appendix-proofs}), delimits their scope, and relates them to two neighboring frameworks.

\subsection{Deferred proofs}\label{appendix-proofs}

\subsubsection{Proof of \cref{lem:readout}}
Sufficiency is an induction on $T$ for the decision, together with the observation that a constant gap between readout values forces, through uniform continuity, a uniform gap between memory states. Necessity constructs the readout: optimality at every horizon makes the action constant along each orbit, cost independence keeps the orbits a positive distance apart, and a truncated distance to each orbit is then a uniformly continuous readout that is invariant and separating.

\begin{proof}[Proof of \cref{lem:readout}]
  \textbf{($\Rightarrow$) Necessity.}
  Suppose $\pi$ is delay-horizon independent. By \ref{a:dec} the optimal action does not
  depend on $T$, so for every $\sigma$ and every $n \geq T_{\min}$,
  \[
    \pi(\idle^{n}(h_0(\sigma))) = a^*(\sigma, o_T) :
  \]
  the action is constant along the orbit $\cR_\sigma$, and orbits seeded by cues with
  different optimal actions are disjoint. Cost independence gives $\sep \geq \mu > 0$, so
  those orbits are at distance at least $\mu$. Enumerate the optimal-action classes
  $1,\dots,m$, let $O_k$ be the union of the orbits in class $k$, and set
  \[
    \readout(h)_k \coloneqq \max(0, 1 - \tfrac{3}{\mu}\,\dist(h, O_k)),
    \qquad k = 1,\dots,m .
  \]
  Each coordinate is a truncated distance function, so $\readout$ is Lipschitz with constant
  $3/\mu$ and in particular uniformly continuous. On $O_k$ it equals the basis vector $e_k$,
  since every other class lies at distance at least $\mu > \mu/3$. It is therefore constant
  on each orbit, hence invariant on $\cR$, and separating because $e_i \neq e_j$ for
  $i \neq j$. Taking $\bar\pi(e_k)$ to be the optimal action of class $k$ gives
  $\pi = \bar\pi \circ \readout$ on $\cR$.

  The construction is deliberate: necessity cannot be established for an \emph{arbitrary}
  factorization, because $\bar{\pi}$ is in general not injective, so equality of actions
  along an orbit constrains $\readout(\idle(h))$ and $\readout(h)$ only to lie in the same fibre
  of $\bar{\pi}$, not to coincide. This is why \cref{lem:readout} asserts the existence of a
  factorization rather than a property of every factorization.

  \textbf{($\Leftarrow$) Sufficiency.}
  If $\readout$ is invariant on $\cR$, then by induction $\readout(\idle^n(h)) = \readout(h)$ for
  every $n \geq 0$ and every $h \in \cR$. Hence
  $\readout(\idle^{T}(h_0(\sigma))) = \readout(h_0(\sigma))$ for all
  $T \geq T_{\min}$: the representation reaching $\bar{\pi}$ at the decision step is
  independent of the horizon. Since $\readout$ is separating and $\pi$ is optimal at
  $T_{\min}$, $\bar{\pi}$ maps that representation to $a^*(\sigma,o_T)$ at every horizon, which
  is decision independence.

  For cost independence, the memory is $\cH$ and one application of $\Phi$ is performed per
  step, neither of which depends on $T$. For the requirement on precision, let
  $\Delta \coloneqq \min \|\readout(h_0(\sigma))-\readout(h_0(\sigma'))\|$ over the pairs with
  $[\sigma] \neq [\sigma']$. The minimum is attained because $\Sigma$ is finite,
  and it is positive by separation. Uniform continuity then yields a $\mu>0$ such that
  $\|h-h'\|<\mu$ implies $\|\readout(h)-\readout(h')\|<\Delta$. Since $\readout$ is constant
  on each orbit with values differing by at least $\Delta$ across classes, every pair of such
  orbits is at distance at least $\mu$, so $\sep \geq \mu > 0$.
\end{proof}

\subsubsection{Proof of \cref{thm:mono}}
\begin{proof}[Proof of \cref{thm:mono}]
Monostability gives
$\idle^n(h)\to h^\star$ for every $h\in\cH$, so every orbit $\cR_\sigma$
accumulates at $h^\star$.
For any pair $\sigma,\sigma'$ requiring different actions and any $n$,
\[
  \big\|\idle^n(h_0(\sigma))-\idle^n(h_0(\sigma'))\big\|
  \le \big\|\idle^n(h_0(\sigma))-h^\star\big\| + \big\|h^\star-\idle^n(h_0(\sigma'))\big\|
  \xrightarrow[n\to\infty]{} 0,
\]
so the infimum defining $\sep$ is $0$ and cost independence fails.
\end{proof}

\subsubsection{Proof of \cref{cor:tstar}}
\begin{proof}[Proof of \cref{cor:tstar}]
Let $\sigma,\sigma'$ be cues with $[\sigma]\neq[\sigma']$. Both orbits converge to $h^\star$, so there is an $n\geq T_{\min}$ such that $\|\idle^n(h_0(\sigma))-\idle^n(h_0(\sigma'))\|<\eta$. The policy then takes the same action on both states. One of the two actions is not optimal, so $\pi$ is not optimal at the horizon $n$, and $T^\star(\pi)<n$.
\end{proof}

\subsubsection{Proof of \cref{thm:multi}}
\begin{proof}[Proof of \cref{thm:multi}]
  Fix $\sigma$ and let $k = k(\sigma)$. Basins of attraction are forward invariant, so
  $\idle^T(h_0(\sigma)) \in B_k$ for all $T \geq 0$. By assumption (ii), $\readout$ is invariant
  within $B_k$, so
  $\readout(\idle^{T+1}(h_0(\sigma))) = \readout(\idle^T(h_0(\sigma)))$
  for all $T \geq 0$, and the policy therefore produces the same action at every horizon.
  Since $\pi$ is optimal at $T_{\min}$, it is then optimal for every $\sigma$ and every $T \geq T_{\min}$, which is decision independence. 
  For cost independence, the memory is $\cH$ and one application of $\Phi$ is
  performed per step, neither of which depends on $T$. For the requirement on precision, let
  $\sigma,\sigma'$ be cues requiring different actions, so $i = k(\sigma) \neq k(\sigma') = j$.
  The orbits $\cR_\sigma \subseteq B_i$ and $\cR_{\sigma'} \subseteq B_j$ have disjoint
  closures, since basins are disjoint and forward invariant and the orbits converge to
  $A_i$ and $A_j$ respectively. Disjoint compact sets are at positive distance, and there are finitely many pairs of cues, so $\sep > 0$.
\end{proof}

\subsubsection{Proof of \cref{prop:nonconst}}
\begin{proof}[Proof of \cref{prop:nonconst}]
  By forward invariance, $h_t \in B_{k(\sigma)}$ for all $t \geq 0$: the
  memory state never leaves the region in which the key observation placed it, whatever the
  sequence of idle-phase observations and actions. By the invariance of $\readout$ within
  $B_k$ under every $\Phi_x$,
  \[
    \readout(h_{T})
    = \readout(\Phi_{x_T}(h_{T-1}))
    = \readout(h_{T-1})
    = \cdots
    = \readout(h_0(\sigma)).
  \]
  The representation reaching $\bar\pi$ at the decision step is therefore independent of the
  horizon and of the idle-phase trajectory.
  Cost independence follows as in \cref{thm:multi}: cues with different optimal actions occupy regions at positive distance, so $\sep \geq \min_{i\neq j}\dist(B_i,B_j)>0$.
\end{proof}

\subsubsection{Proof of \cref{thm:robust}}
\begin{proof}[Proof of \cref{thm:robust}]
  Let $\sigma_1,\sigma_2$ be cues with $[\sigma_1] \neq [\sigma_2]$, and
  let $\cR_i$ denote the closure of the set of states reachable from $h_0(\sigma_i)$ under
  arbitrary compositions of maps in $\cW_\eta$.

  \emph{Step 1.} Any particular composition is a realization of the idle phase in
  $\cW_\eta$, along which $\readout$ is invariant by hypothesis. Hence $\readout$ takes the value
  $\readout(h_0(\sigma_i))$ on the reachable set of every realization, and by continuity the same value
  $c_i$ on $\cR_i$. Let $\Delta = \|c_1-c_2\|$, positive by separation.

  \emph{Step 2.} Uniform continuity gives a $\mu>0$ such that
  $\|h-h'\|<\mu$ implies $\|\readout(h)-\readout(h')\|<\Delta$. Hence
  $\dist(\cR_1,\cR_2)\geq\mu>0$, so $\cR_1$ and $\cR_2$ are disjoint. Being closed subsets
  of the compact $\cH$, they are compact.

  \emph{Step 3.} Let $h\in\cR_i$ and let $e$ satisfy $\|e\|\leq\eta$ and
  $\idle(h)+e\in\cH$. The map sending $h$ to $\idle(h)+e$ and agreeing with $\idle$ elsewhere lies in
  $\cW_\eta$, so $\idle(h)+e\in\cR_i$ by construction of $\cR_i$ and continuity of $\idle$. Thus
  $\big(\idle(\cR_i) \oplus \bar B_\eta\big) \cap \cH \subseteq \cR_i$: the set $\idle(\cR_i)$
  lies in the relative interior of $\cR_i$ in $\cH$, and $\cR_i$ is a trapping region.

  \emph{Step 4.} For a compact $N$ with $\idle(N)$ in its relative interior and $\idle$ continuous, the set
  $A \coloneqq \bigcap_{n\geq0}\idle^n(N)$ is nonempty, compact and invariant, attracts every
  point of $N$, and has $N$ as a neighborhood in $\cH$. Applying this to $\cR_1$ and $\cR_2$ yields attractors $A_1\subseteq\cR_1$ and
  $A_2\subseteq\cR_2$, disjoint by Step 2, with basins containing $\cR_1$ and $\cR_2$.
\end{proof}

\subsection{Scope of the results}\label{appendix-scope}

\paragraph{Indexing of the idle phase.} The idle map governs the $T-1$ steps of the idle phase. At $t = T$, the memory is updated once more, by $D \coloneqq \Phi_{x_T}$. For each decision observation $o_T$ this is one fixed continuous map, which we absorb into the readout: $\readout \circ D$ is again a readout, continuous if $\readout$ is, and when $o_T$ can take several values, $\bar\pi$ also receives $o_T$. The policy therefore acts on $\idle^{T-1}(h_0(\sigma))$. We write $h_T = \idle^T(h_0(\sigma))$ for readability, since a shift of the iterate index by one leaves every asymptotic statement unchanged. Throughout, our analysis concerns the steps $t = 0, \ldots, T$; post-decision dynamics ($t > T$) are not modeled.

Our two results have different scopes. \Cref{lem:readout}, the invariant-and-separating readout condition, is architecture-agnostic: it holds for any policy that factorizes as $\bar\pi \circ \readout$. \Cref{thm:mono}, the monostability
obstruction, additionally assumes the structure of \cref{eq:update}, namely a
fixed-dimensional, state-autonomous idle map $\idle \colon \cH \to \cH$. This yields a classification
of memory mechanisms rather than a list of exceptions, and the classification is driven by
the two requirements of \cref{def:thi}: a mechanism may fail horizon independence by losing the
decision, by paying a cost that grows with the horizon, or by both.

\paragraph{Attention-based policies.} A policy attending over its full context escapes
\cref{thm:mono}: the analogue of $h_{t-1}$ is the set of past representations, whose size
grows with $T$, so there is no fixed-point map and the obstruction does not apply. A
readout attending to the cue token is invariant and separating, and decision independence
is therefore achievable without multistability. It is not, however, horizon independence.

\begin{corollary}[Attention is not delay-horizon independent]
  \label{cor:attention}
  A policy attending over its full context satisfies decision independence,
  \cref{def:thi}\ref{thi:dec}, but violates cost independence,
  \cref{def:thi}\ref{thi:cost}: its memory footprint grows linearly and its per-step
  attention cost grows at least linearly in $T$. A policy attending over a context window of
  fixed length $k$ restores cost independence but loses decision independence: for $T > k$
  the key observation has left the window, so no readout can separate cues requiring
  different actions, and the effective horizon satisfies $T^\star \leq k$. In neither case is
  the policy delay-horizon independent.
\end{corollary}

Attention therefore does not provide a second mechanism for horizon independence; it
provides a trade between the two requirements. Unbounded context buys the decision at a cost that
grows with the horizon, which is the resource growth the property was defined to exclude;
bounded context buys constant cost at a horizon ceiling of a different origin than fading
memory but with the same practical consequence, namely that the deployment horizon must be
known and bounded at design time. This is consistent with the empirical picture reported by
\citet{ni2023transformers}, where transformer policies sustain memory lengths an order of
magnitude beyond recurrent baselines, and then stop. Together with
\cref{thm:mono,thm:multi}, this leaves multistability on a bounded-dimension state as the
only mechanism satisfying both requirements robustly.

\paragraph{External memory and adaptive memory selection.} Architectures that equip the
agent with an addressable external register \citep{peshkin1999learning} or with a fixed-size
memory stack under learned write and eviction decisions satisfy our assumptions and are
instances of \cref{thm:multi} rather than counterexamples. When the agent emits memory
actions that set and clear the bits of a fixed-size register fed back as input, the update
under the null action is the identity on the register, so each of the $2^b$ register
configurations is a fixed point: this is multistability with $2^b$ attractors, realized by
construction rather than learned, and with no transient component during the idle phase. The
same reading applies to classical finite-state controllers.

\subsection{Horizon generalization and delay-horizon independence}\label{appendix-hg}

\paragraph{Horizon generalization (HG).} \citet{myers2025horizongeneralizationreinforcementlearning}
study fully observed, goal-conditioned MDPs. A policy $\pi(a \mid s, g)$ is commanded with a
goal $g$ that is available at every step, and the horizon is the distance from $s$ to $g$,
measured by a learned quasimetric. Horizon generalization is the property that a policy
optimal for nearby goals remains optimal for distant ones, and it follows from planning
invariance: by the triangle inequality, acting toward a distant goal coincides with acting
toward a nearer waypoint on the way to it. The horizon-extending operator is path
relaxation, which roughly doubles the reachable distance.

Our setting is partially observed. The
horizon is the number of steps between a key observation and a later decision, and the
optimal action is independent of that number (\ref{a:dec}). There is no goal and no metric,
and the policy is not memoryless: it acts on $h_t$, a recurrent summary of the history. The
horizon-extending operator is the idle map $\idle$, which adds one step.

HG assumes the goal is available at every step, so
the information determining the optimal action is never absent from the observation; in our
setting that information is presented once and must be carried. Because HG policies are
memoryless, forgetting cannot occur there, whereas forgetting is the entire subject of DHI.
Furthermore, \citet{myers2025horizongeneralizationreinforcementlearning} explicitly exclude
generalization to unvisited states, whereas in our setting the memory states occupied at
long horizons are, by construction, never visited during training. DHI is therefore not HG
applied to a sub-case. What the two properties share is a common shape: a property holding
at short horizon extends to long horizon when the policy is invariant under the operator
that extends the horizon.

\textbf{Generalization versus independence.} HG carries no resource requirement, which is appropriate for a
memoryless goal-conditioned policy: there is nothing that could grow with the horizon. Our
setting is the opposite case. The whole difficulty is that the information determining the
action must be carried across the idle phase, and a policy can always carry it by storing
more, so a property that constrained only the decision would be satisfied by the history
buffer and would say nothing about memory mechanisms. Requirement \ref{thi:cost} of
\cref{def:thi} closes that door and makes independence different from generalization. Fading-memory models demonstrably generalize over the horizon, over a finite range. They cannot be independent of the horizon in decision
and in cost at once. Using one name for both would force every
statement of our results into the form of a qualified failure, which is precisely the
confusion we wish to avoid.

\textbf{POMDP structure versus general recurrent memory.} We separate what in
our results comes from the POMDP structure and what comes from the memory itself.
\Cref{thm:mono} is a statement about the recurrent memory map and holds for any
fixed-dimensional recurrent state. The delayed-cue structure of \cref{def:class} makes the readout time-invariant (through \ref{a:idle}) and makes the optimal action depend
on $\sigma$ (through \ref{a:dec}), thereby turning the dynamical property into a
statement about optimality.

\subsection{Relation to state abstraction}\label{appendix-abstraction}

\Cref{def:invariant} requires $\readout$ to separate cues only up to optimal-action
equivalence. The induced partition of the cue set is exactly a $\pi^*$-irrelevance
abstraction in the taxonomy of \citet{li2006towards}, the coarsest of the five abstractions
considered there: it constrains only which cues induce distinct optimal actions, preserves
no value function, is strictly coarser than $a^*$-irrelevance (which additionally preserves
$Q^*$ at the optimal action), and is not model-preserving.

The caveat that \citet{li2006towards} attach to $\phi_{\pi^*}$, namely that value-based
learning under this abstraction may converge to a policy that is suboptimal in the ground
MDP, is explicitly stated there not to preclude policy-search methods from recovering
optimality. Our agents are trained with PPO, which learns $\pi$ and $V_\pi$ directly and
never forms $Q$, so the caveat does not bind in our experiments.

The two frameworks answer complementary questions. State abstraction is static: it asks
which state distinctions may be discarded without losing optimality. We ask a dynamical
question: which recurrent dynamics can carry a retained distinction forward over an
unbounded number of steps. The invariance condition on $\idle$ makes a
$\pi^*$-irrelevance readout realizable at every horizon rather than only at the horizons
seen during training.

\newpage
\section{Recurrent Architectures: Dynamics and Steady States}
\label{appendix-RNNdyn}

This appendix collects the architectural detail summarized in
\cref{section-setting,section-mono,section-multi}: the recurrent families we consider, the
steady-state computations behind the monostability and multistability claims, and the
specific cells used in the experiments. Their update equations are listed in
\cref{appendix-RNNeq}.

\subsection{Recurrent memory}
RNNs encode time-dependencies in a recurrent hidden state $h_t$ whose update depends on the previous hidden state $h_{t-1}$ and current observation of a time-series $x_t$. It writes
\begin{equation}\label{eq:RNNapp}
    h_t = \Phi(h_{t-1},x_t), \qquad \forall t \geq 1,
\end{equation}
where $\Phi$ is the update map of \cref{eq:update}, which depends on the learnable parameters $\theta$ of the network. Nonlinear versions such as LSTMs~\citep{hochreiter1997longshorttermmemory} and GRUs~\citep{cho2014learningphraserepresentationsusing} have dominated sequence modeling for decades, but many variants such as Janet~\citep{vanderwesthuizen2018unreasonableeffectivenessforgetgate}, (n)BRC~\citep{vecoven2021bioinspiredbistablerecurrentcell}, and MGU~\citep{zhou2016minimalgatedunitrecurrent} have appeared over the years. Due to their recurrent nature, RNNs can encode time dependencies in their time-varying dynamics (\cref{eq:RNNapp}), creating fading memory of past inputs.

Nonlinear RNNs can also exhibit multistability, the coexistence of multiple possible state values $\bar{h}$ at equilibrium for the same constant input $\bar{x}$. An RNN is multistable if its steady-state equation
\begin{equation}
    \bar{h} - \Phi(\bar{h},\bar{x})=0,
\end{equation}
admits multiple solutions for some set of constant input $\bar{x}$. Multistability allows for persistent memory under constant input value through a history-dependent convergence toward one of the multiple equilibria~\citep{vecoven2021bioinspiredbistablerecurrentcell,lambrechts2023warmingrecurrentneuralnetworks}. Most nonlinear RNNs can reach multistability for some set of learnable parameters~$\theta$~\citep{lambrechts2023warmingrecurrentneuralnetworks}. In this work, we use the GRU, BRC and nBRC variants in our experiments (\cref{appendix-RNNeq}).

The converse reading is equally informative. Widely deployed nonlinear cells such as the GRU and the LSTM retain state-dependent gating and can therefore realize multistability; nothing in our results excludes them from DHI. Our experiments show (\cref{section-multi,appendix-supFig}) that multistability is reached only by a fraction of trained models, so that transferring these results to GRU or LSTM policies is a question of \emph{training} procedures that foster multistability, not of architectural capacity.

\subsection{Parallelizable architectures are monostable}
Recent efforts have led to the emergence of novel RNN architectures that are parallelizable over the sequence length with competitive performance on sequential benchmarks. They are, without exception among the families below, strictly monostable. By \cref{thm:mono}, these architectures are therefore excluded from DHI \emph{at every parameter value}, and have a finite effective horizon whatever the training procedure. Crucially, this verification is carried out in closed form on the update equation, without training a single model.

SSMs and linear RNNs exploit linear state dynamics to achieve parallelizable training, either through convolutional representations \citep{gu2021combiningrecurrentconvolutionalcontinuoustime, gu2022efficientlymodelinglongsequences} or the parallel scan algorithm \citep{blelloch1989scans, martin2018parallelizinglinearrecurrentneural, gu2024mambalineartimesequencemodeling}. The update equation of arbitrary linear RNNs writes $$h_t = A(x_t)h_{t-1}+B(x_t)x_t,$$ where stability requires all eigenvalues of $A(x_t)$ to be contained within the unit circle in the complex plane. The steady-state solutions $$\bar{h}=-(A(\bar{x})-I)^{-1}B(\bar{x})\bar{x}$$
are explicit, meaning that only one \textit{and only one} solution $\bar{h}$ exists for any input $\bar{x}$. Under the constant idle input $\bar{x}$, the idle map is the single affine map $\idle(h) = A(\bar{x})h + B(\bar{x})\bar{x}$, and when $\rho(A(\bar{x}))<1$ its fixed point is globally attracting, since $\|\idle^n(h)-\bar{h}\| = \|A(\bar{x})^n(h-\bar{h})\| \to 0$. The idle map $\idle$ of such a model is therefore monostable in the sense of \cref{def:mono}, and \cref{thm:mono} applies.

This argument covers selective SSMs such as Mamba \citep{gu2024mambalineartimesequencemodeling} directly: their state update is linear in the state, with parameters that are input-dependent but \emph{not} state-dependent, so it is exactly of the form $h_t = A(x_t)h_{t-1}+B(x_t)x_t$. Under the constant idle input of \cref{eq:U}, $A$ and $B$ are constant and the steady state is unique. Selectivity modulates the \emph{rate} at which information decays, and can make that rate very small, but it cannot create a second attractor: it lengthens $T^\star$ without making it infinite.

Gated RNNs can also achieve parallelizable training by removing the state dependency in their gating mechanisms, \textit{i.e.}, computing gate values from the input alone, $z_t = \sigma(W_z x_t)$. Several such architectures have been proposed, including xLSTM \citep{beck2024xlstmextendedlongshortterm}, HGRN \citep{qin2023hierarchicallygatedrecurrentneural}, MatMul-free LM \citep{zhu2025scalablematmulfreelanguagemodeling}, GILR \citep{martin2018parallelizinglinearrecurrentneural}, and minGRU \citep{feng2024weregrudminimalgatedrecurrent}, though some only achieve partial parallelization. It yields an update equation of the form $$h_t = (1 - z_t) \odot h_{t-1} + z_t \odot \tilde{h}_t,$$
where $\tilde{h}_t$ depends only on $x_t$. The steady-state solutions reduce to $$\bar{h} = \tilde{h} = \phi(W_h \bar{x})$$
where $\phi$ is an activation function whose structure is determined by the dynamical equation of $\tilde{h}_t$. The steady-state equation maps an explicit solution $\bar{h}$ to any input value $\bar{x}$, so these models are monostable as well. This result extends to gated RNNs with multiple gates, provided that no gate is state-dependent. In this work, we use minGRU as a representative linear gated RNN (\cref{appendix-RNNeq}).

\subsection{Memory recurrent units and the BMRU}
Memory recurrent units (MRUs) represent a class of parallelizable RNNs with properties complementary to those of linear RNNs: they exhibit multistability but lack transient dynamics (\textit{i.e.},  they assume instantaneous convergence) \citep{degeeter2026parallelizable, brandoit2026improving}. At each timestep $t$ and for a given input $x_t$, the MRU outputs a solution to the steady-state equation $$\bar{h}_t - \Phi(\bar{h}_t,x_t) = 0.$$ When multiple solutions exist, the previous state $\bar{h}_{t-1}$ determines the selected solution, effectively modeling history-dependent behavior. A member of this family is the bistable memory recurrent unit (BMRU), which integrates learnable hysteretic dynamics into a gated architecture \citep{degeeter2026parallelizable}. In this work, we use the BMRU as a representative multistable model to contrast with the monostable nature of modern gated linear RNNs (\cref{appendix-RNNeq}).

\newpage
\section{T-maze Environment}
\label{appendix-tmaze}

We base our formal definition of the T-maze on the work of \citet{lambrechts2023warmingrecurrentneuralnetworks}. The T-maze environment is a POMDP $(\mathcal{S}, \mathcal{A}, \mathcal{O}, p_0, P, \Omega, R, \gamma)$ parameterized by the maze length $L \in \mathbb{Z}^+$.

\textbf{State space}\hspace{0.5cm} The discrete state space $\mathcal{S}$ is composed of the position of the agent $\mathcal{P}$ and the position of the goal $\mathcal{G}$, with

\[
\left\{
\begin{aligned}
    \mathcal{S} &= \mathcal{P} \times \mathcal{G}\\
    \mathcal{P} &= \left\{(0, 0), \ldots, (L-1, 0)\right\} \cup \left\{(L-1, -1), (L-1, 1)\right\}\\
    \mathcal{G} &= \left\{ -1, 1 \right\},
\end{aligned}
\right.
\]

where the -1 and 1 goal positions represent up and down respectively. Additionally, there are four terminal states given by the set $\mathcal{F} = \left\{s_t = (p_t, g_t) \in \mathcal{S} \mid p_t \in \left\{(L-1, -1), (L-1, 1)\right\} , g_t \in \mathcal{G}\right\}$.

\textbf{Action space}\hspace{0.5cm} The discrete action space is composed of the four possible directions in which the agent can move and is defined as $\mathcal{A} = \{0, 1, 2, 3\},$ where the values correspond to the right, top, left, and down directions respectively.

\textbf{Observation space}\hspace{0.5cm} The discrete observation space is composed of the initial observation of the goal at the first timestep, 0 otherwise, and a flag indicating whether the agent has reached the far right of the maze (1) or not (0). Formally, we have $$\mathcal{O} = \left\{(g_t, f_t) \mid g_{t=0} \in \mathcal{G}, g_{t >0} = 0, f_t \in \left\{1, 0\right\}\right\}.$$

\textbf{Initial state distribution}\hspace{0.5cm} The two possible initial states are $s_0^{up} = ((0, 0), -1)$ and $s_0^{down} = ((0, 0), 1)$. The initial state distribution $p_0 \colon \mathcal{S} \to \left[0, 1\right]$ is given by
$$
p_0(s_0) =
\begin{cases}
    0.5 & \text{if } s_0 = s_0^{up},\\
    0.5 & \text{if } s_0 = s_0^{down},\\
    0 & \text{otherwise}.
\end{cases}
$$

\textbf{Transition distribution}\hspace{0.5cm} The transition distribution function $P \colon \mathcal{S} \times \mathcal{A} \times \mathcal{S} \to \left[0, 1\right]$ is given by $$P(s_{t+1} \mid s_t, a_t) = \delta_{f(s_t, a_t)} (s_{t+1}),$$ where $s_{t+1} \in \mathcal{S}$, $s_t \in \mathcal{S}$ and $a_t \in \mathcal{A}$, and $f$ is given by
$$
f(s_t, a_t) =\\
\begin{cases}
    s_{t+1} = (p_t + (1, 0), g_t) & \text{if } s_t \notin \mathcal{F}, p_t + (1, 0) \in \mathcal{P}, a_t = 0,\\
    s_{t+1} = (p_t - (1, 0), g_t) & \text{if } s_t \notin \mathcal{F}, p_t - (1, 0) \in \mathcal{P}, a_t = 2,\\
    s_{t+1} = (p_t + (0, 1), g_t) & \text{if } s_t \notin \mathcal{F}, p_t + (0, 1) \in \mathcal{P}, a_t = 3,\\
    s_{t+1} = (p_t - (0, 1), g_t) & \text{if } s_t \notin \mathcal{F}, p_t - (0, 1) \in \mathcal{P}, a_t = 1,\\
    s_{t+1} = (p_t, g_t) & \text{otherwise},
\end{cases}
$$
where $s_t = (p_t, g_t) \in \mathcal{S}$ and $a_t \in \mathcal{A}$.

\textbf{Reward function}\hspace{0.5cm} The reward function $R \colon \mathcal{S} \times \mathcal{A} \times \mathcal{S} \to \mathbb{R}$ is given by
$$
R(s_t, a_t, s_{t+1}) =\\
\begin{cases}
    4 & \text{if } s_t \notin \mathcal{F}, s_{t+1} \in \mathcal{F}, p_{t+1} = (L-1, g_0),\\
    -0.1 & \text{if } s_t \notin \mathcal{F}, s_{t+1} \in \mathcal{F}, p_{t+1} = (L-1, r), r \in \mathcal{G}, r \neq g_0\\
    0 & \text{otherwise},
\end{cases}
$$
where $s_t = (p_t, g_t) \in \mathcal{S}$, $a_t \in \mathcal{A}$, $s_{t+1} = (p_{t+1}, g_{t+1}) \in \mathcal{S}$.

\textbf{Observation distribution}\hspace{0.5cm} The observations are deterministic. The observation distribution $\Omega \colon \mathcal{O} \times \mathcal{S} \to \left[0, 1\right]$ is given by
$$
\Omega(o_t \mid s_t) =
\begin{cases}
    1 & \text{if } o_t = (g_t, \delta(p_t = L-1)),\\
    1 & \text{if } o_t = (0, \delta(p_t = L-1))),\\
    0 & \text{otherwise},
\end{cases}
$$
where $o_t \in \mathcal{O}$ and $s_t = (p_t, g_t) \in \mathcal{S}$.


\newpage
\section{MiniGrid-Memory Environment}
\label{appendix-memory}

The MiniGrid-Memory \citep{MinigridMiniworld23} environment is a POMDP $(\mathcal{S}, \mathcal{A}, \mathcal{O}, p_0, P, \Omega, R, \gamma)$ parameterized by a finite set of possible maze sizes $\mathcal{N} \subset \mathbb{Z}^+$ (odd, $n \geq 7$) and a view distance $v \in \mathbb{Z}^+$ (odd, $v \geq 3$). The maze is composed of a $3 \times 3$ start room connected by a straight hallway to a goal hallway containing two candidate goal cells, one of which is rewarding. Which goal is rewarding is cued to the agent by a colored marker placed at a fixed cell in the start room, so the agent must remember this cue while traversing the hallway in order to select the correct goal.

\textbf{State space}\hspace{0.5cm} The discrete state space $\mathcal{S}$ is composed of the maze size $n$, the agent position $p$ and heading $\delta$, the index $k$ of the cued (rewarding) goal, and the goal layout $L$, with

\[
\mathcal{S} = \left\{ (n, p, \delta, k, L) \;\middle|\;
\begin{aligned}
    &n \in \mathcal{N}, \ \delta \in \{0,1,2,3\}, \ k \in \{0,1\}, \\
    &p \in \mathcal{P}(n), \ L \colon \{0,1\} \to \mathcal{K}(n) \text{ a bijection}
\end{aligned}
\right\},
\]

where, writing $c(n) = \lfloor n/2 \rfloor$,

\[
\left\{
\begin{aligned}
    \mathcal{K}(n) &= \left\{ (n-2, c(n)-2), (n-2, c(n)+2) \right\}\\
    \mathcal{P}(n) &= \{1,2,3\} \times \{c(n)-1, c(n), c(n)+1\} \\
        &\quad \cup\ \{4,\ldots,n-3\} \times \{c(n)\} \\
        &\quad \cup\ \{n-2\} \times \{1,\ldots,n-2\}
\end{aligned}
\right.
\]

The three sets composing $\mathcal{P}(n)$ are, respectively, the $3 \times 3$ start room, the horizontal hallway, and the vertical goal hallway; together they form the set of reachable cells of the maze. $\mathcal{K}(n)$ is the set of two candidate goal cells, located at the top and bottom of the goal hallway. The bijection $L$ assigns a goal index $i \in \{0,1\}$ to the cell $L(i) \in \mathcal{K}(n)$ it occupies for the episode, and $k$ is the index of the rewarding goal, cued to the agent at the fixed cell $q_{\mathrm{cue}}(n) = (1, c(n)-1)$ in the start room (see Observation distribution below). There are terminal states given by the set $\mathcal{F} = \left\{ s_t = (n_t, p_t, \delta_t, k_t, L_t) \in \mathcal{S} \mid p_t \in \mathcal{K}(n_t) \right\}$, reached whenever the agent occupies either goal cell.

\textbf{Action space}\hspace{0.5cm} The discrete action space is composed of the three actions available to the agent and is defined as $\mathcal{A} = \{0, 1, 2\}$, where each value corresponds to turning left, turning right, and moving forward respectively.

\textbf{Observation space}\hspace{0.5cm} The discrete observation space is composed of an egocentric local window of the cells surrounding the agent and the agent current heading. Formally, we have $\mathcal{O} = \mathcal{O}_c \times \mathcal{O}_\delta$ with $\mathcal{O}_c = \{0, 1, 2, 3\}^{v \times v}$ and $\mathcal{O}_\delta = \{0, 1, 2, 3\}$.

\textbf{Initial state distribution}\hspace{0.5cm} At the start of an episode, the maze size is drawn uniformly, $n \sim \mathcal{U}(\mathcal{N})$; the goal layout is drawn uniformly among the $2! = 2$ assignments of goal indices to cells, $L \sim \mathcal{U}(\mathfrak{S}(\mathcal{K}(n)))$; the cued goal index is drawn uniformly, $k \sim \mathcal{U}(\{0,1\})$; the agent starts facing right, $\delta_0 = 0$, at a uniformly random cell along the maze central row, $p_0 = (x_0, c(n))$ with $x_0 \sim \mathcal{U}(\{1, \ldots, n-2\})$. The initial state distribution $p_0 \colon \mathcal{S} \to [0, 1]$ is given by
$$
p_0(s_0) =
\begin{cases}
    \dfrac{1}{|\mathcal{N}|} \cdot \dfrac{1}{n-2} \cdot \dfrac{1}{2} \cdot \dfrac{1}{2} & \text{if } s_0 = \left(n, (x_0, c(n)), 0, k, L\right), \\
    0 & \text{otherwise},
\end{cases}
$$
for $n \in \mathcal{N}$, $x_0 \in \{1, \ldots, n-2\}$, $k \in \{0,1\}$, $L \in \mathfrak{S}(\mathcal{K}(n))$.

\textbf{Transition distribution}\hspace{0.5cm} The transition distribution function $P \colon \mathcal{S} \times \mathcal{A} \times \mathcal{S} \to [0, 1]$ is deterministic and given by $$P(s_{t+1} \mid s_t, a_t) = \delta_{f(s_t, a_t)}(s_{t+1}),$$ where $s_{t+1} \in \mathcal{S}$, $s_t = (n, p_t, \delta_t, k, L) \in \mathcal{S} \setminus \mathcal{F}$, $a_t \in \mathcal{A}$, and $f(s_t, a_t) = (n, p_{t+1}, \delta_{t+1}, k, L)$ with
$$
\delta_{t+1} =
\begin{cases}
    (\delta_t - 1) \bmod 4 & \text{if } a_t = 0,\\
    (\delta_t + 1) \bmod 4 & \text{if } a_t = 1,\\
    \delta_t & \text{if } a_t = 2,
\end{cases}
\qquad
p_{t+1} =
\begin{cases}
    p_t + e(\delta_t) & \text{if } a_t = 2 \text{ and } p_t + e(\delta_t) \in \mathcal{P}(n),\\
    p_t & \text{otherwise},
\end{cases}
$$
and $e \colon \{0,1,2,3\} \to \mathbb{Z}^2$ maps a heading to its unit direction vector, $e(0) = (1, 0)$, $e(1) = (0, 1)$, $e(2) = (-1, 0)$, $e(3) = (0, -1)$ (right, down, left, up respectively). That is, turning changes the agent heading without moving it, and moving forward advances the agent by one cell along its current heading unless that cell lies outside $\mathcal{P}(n)$ (a wall), in which case the agent remains in place.

\textbf{Reward function}\hspace{0.5cm} The reward function $R \colon \mathcal{S} \times \mathcal{A} \times \mathcal{S} \to \mathbb{R}$ is given by
$$
R(s_t, a_t, s_{t+1}) =
\begin{cases}
    4 & \text{if } s_t \notin \mathcal{F}, \ p_{t+1} = L(k),\\
    -0.1 & \text{if } s_t \notin \mathcal{F}, \ p_{t+1} = L(1-k),\\
    0 & \text{otherwise},
\end{cases}
$$
where $s_t = (n, p_t, \delta_t, k, L) \in \mathcal{S}$, $a_t \in \mathcal{A}$, $s_{t+1} = (n, p_{t+1}, \delta_{t+1}, k, L) \in \mathcal{S}$. A reward of $4$ is given for entering the cued goal $L(k)$, a reward of $-0.1$ for entering the other goal $L(1-k)$, and reaching either goal ends the episode.

\textbf{Observation distribution}\hspace{0.5cm} The observations are deterministic. Writing $s_t = (n, p_t, \delta_t, k, L) \in \mathcal{S}$, the observation $o_t = (c_t, \delta_t) \in \mathcal{O}$ is given by $\Omega(o_t \mid s_t) = 1$. The heading component of the observation is simply $\delta_t$, i.e. the agent heading is fully observed at every timestep. The local window $c_t \in \{0, \ldots, 3\}^{v \times v}$ is centered $d = \lfloor v/2 \rfloor$ cells ahead of the agent along its heading, covering the world cells $\left\{ p_t + d \cdot e(\delta_t) + \omega \mid \omega \in \{-d, \ldots, d\}^2 \right\}$, each labeled by the labeling function $\kappa(\cdot\,; n, k, L) \colon \mathbb{Z}^2 \to \{0, \ldots, 3\}$,
$$
\kappa(q; n, k, L) =
\begin{cases}
    2 + i & \text{if } q = L(i) \text{ for some } i \in \{0,1\},\\
    2 + k & \text{if } q = q_{\mathrm{cue}}(n),\\
    1 & \text{else if } q \notin \mathcal{P}(n),\\
    0 & \text{otherwise},
\end{cases}
$$
i.e. $0$ denotes an empty cell, $1$ a wall, and $2+i$ an object of color $i$, either an actual goal cell or, at the fixed cue cell $q_{\mathrm{cue}}(n)$, the marker indicating which color is rewarding this episode. The window is expressed in an \emph{egocentric} frame: it is rotated according to the agent current heading $\delta_t$ so that the agent forward direction always maps to the same fixed array axis, irrespective of its true heading in the maze.

\newpage
\section{MiniGrid-GoToDoor Environment}
\label{appendix-gotodoor}

The MiniGrid-GoToDoor \citep{MinigridMiniworld23} environment is a POMDP $(\mathcal{S}, \mathcal{A}, \mathcal{O}, p_0, P, \Omega, R, \gamma)$ parameterized by a finite set of possible room sizes $\mathcal{N} \subset \mathbb{Z}^+$ and a view distance $v \in \mathbb{Z}^+$ (odd, $v \geq 3$). The agent navigates an $n \times n$ room ($n \in \mathcal{N}$, resampled every episode) with four doors, one at the midpoint of each wall; one door is designated as the goal. The identity (index) of each door is randomized independently of its location every episode, so the agent must explore the room, observe the doors it encounters, and match their index against the goal index given at the start of the episode.

\textbf{State space}\hspace{0.5cm} The discrete state space $\mathcal{S}$ is composed of the room size $n$, the agent position $p$ and heading $\delta$, the door layout $D$, and the goal door index $g$, with

\[
\mathcal{S} = \left\{ (n, p, \delta, D, g) \;\middle|\;
\begin{aligned}
    &n \in \mathcal{N}, \ \delta \in \{0,1,2,3\}, \ g \in \{0,1,2,3\}, \\
    &p \in \mathcal{P}(n), \ D \colon \{0,1,2,3\} \to \mathcal{W}(n) \text{ a bijection}
\end{aligned}
\right\},
\]

where

\[
\left\{
\begin{aligned}
    \mathcal{W}(n) &= \left\{ (0, \lfloor n/2 \rfloor), (\lfloor n/2 \rfloor, n-1), (n-1, \lfloor n/2 \rfloor), (\lfloor n/2 \rfloor, 0) \right\}\\
    \mathcal{P}(n) &= \left\{1, \ldots, n-2\right\}^2 \cup \mathcal{W}(n)
\end{aligned}
\right.
\]

$\mathcal{W}(n)$ is the set of four door slots, located at the midpoint of the west, south, east and north walls respectively, and $\mathcal{P}(n)$ is the set of reachable cells, i.e. the interior of the room together with the four door slots. The bijection $D$ assigns a door index $i \in \{0,1,2,3\}$ to the slot $D(i) \in \mathcal{W}(n)$ it occupies for the episode, and $g$ is the index of the rewarding (goal) door, located at $D(g)$. Additionally, there are terminal states given by the set $\mathcal{F} = \left\{ s_t = (n_t, p_t, \delta_t, D_t, g_t) \in \mathcal{S} \mid p_t \in \mathcal{W}(n_t) \right\}$, reached whenever the agent occupies any door slot. Independently of $\mathcal{F}$, an episode is truncated whenever $t = T$.

\textbf{Action space}\hspace{0.5cm} The discrete action space is composed of the three actions available to the agent and is defined as $\mathcal{A} = \{0, 1, 2\}$, where each value corresponds to turning left, turning right, and moving forward respectively.

\textbf{Observation space}\hspace{0.5cm} The discrete observation space is composed of the initial observation of the goal door index at the first timestep, $0$ otherwise, an egocentric local window of the cells surrounding the agent, and the agent current heading. Formally, we have $\mathcal{O} = \mathcal{O}_g \times \mathcal{O}_c \times \mathcal{O}_\delta$ with $\mathcal{O}_g = \{0, 1, 2, 3, 4\}$, $\mathcal{O}_c = \{0, \ldots, 5\}^{v \times v}$, and $\mathcal{O}_\delta = \{0, 1, 2, 3\}$.

\textbf{Initial state distribution}\hspace{0.5cm} At the start of an episode, the room size is drawn uniformly, $n \sim \mathcal{U}(\mathcal{N})$; the door layout is drawn uniformly among the $4! = 24$ assignments of door indices to slots, $D \sim \mathcal{U}(\mathfrak{S}(\mathcal{W}(n)))$; the goal door index and initial heading are drawn uniformly, $g \sim \mathcal{U}(\{0,1,2,3\})$ and $\delta_0 \sim \mathcal{U}(\{0,1,2,3\})$; and the agent starts at the center of the room, $p_0 = (\lfloor n/2 \rfloor, \lfloor n/2\rfloor)$. The initial state distribution $p_0 \colon \mathcal{S} \to [0, 1]$ is given by
$$
p_0(s_0) =
\begin{cases}
    \dfrac{1}{|\mathcal{N}|} \cdot \dfrac{1}{24} \cdot \dfrac{1}{4} \cdot \dfrac{1}{4} & \text{if } s_0 = \left(n, (\lfloor n/2 \rfloor, \lfloor n/2 \rfloor), \delta, D, g\right), \\
    0 & \text{otherwise},
\end{cases}
$$
for $n \in \mathcal{N}$, $\delta \in \{0,1,2,3\}$, $D \in \mathfrak{S}(\mathcal{W}(n))$, $g \in \{0,1,2,3\}$.

\textbf{Transition distribution}\hspace{0.5cm} The transition distribution function $P \colon \mathcal{S} \times \mathcal{A} \times \mathcal{S} \to [0, 1]$ is deterministic and given by $$P(s_{t+1} \mid s_t, a_t) = \delta_{f(s_t, a_t)}(s_{t+1}),$$ where $s_{t+1} \in \mathcal{S}$, $s_t = (n, p_t, \delta_t, D, g) \in \mathcal{S} \setminus \mathcal{F}$, $a_t \in \mathcal{A}$, and $f(s_t, a_t) = (n, p_{t+1}, \delta_{t+1}, D, g)$ with
$$
\delta_{t+1} =
\begin{cases}
    (\delta_t - 1) \bmod 4 & \text{if } a_t = 0,\\
    (\delta_t + 1) \bmod 4 & \text{if } a_t = 1,\\
    \delta_t & \text{if } a_t = 2,
\end{cases}
\qquad
p_{t+1} =
\begin{cases}
    p_t + e(\delta_t) & \text{if } a_t = 2 \text{ and } p_t + e(\delta_t) \in \mathcal{P}(n),\\
    p_t & \text{otherwise},
\end{cases}
$$
and $e \colon \{0,1,2,3\} \to \mathbb{Z}^2$ maps a heading to its unit direction vector, $e(0) = (1, 0)$, $e(1) = (0, 1)$, $e(2) = (-1, 0)$, $e(3) = (0, -1)$ (right, down, left, up respectively). That is, turning changes the agent heading without moving it, and moving forward advances the agent by one cell along its current heading unless that cell lies outside $\mathcal{P}(n)$ (a wall), in which case the agent remains in place.

\textbf{Reward function}\hspace{0.5cm} The reward function $R \colon \mathcal{S} \times \mathcal{A} \times \mathcal{S} \to \mathbb{R}$ is given by
$$
R(s_t, a_t, s_{t+1}) =
\begin{cases}
    1 & \text{if } s_t \notin \mathcal{F}, \ s_{t+1} \in \mathcal{F}, \ p_{t+1} = D(g),\\
    0 & \text{otherwise},
\end{cases}
$$
where $s_t = (n, p_t, \delta_t, D, g) \in \mathcal{S}$, $a_t \in \mathcal{A}$, $s_{t+1} = (n, p_{t+1}, \delta_{t+1}, D, g) \in \mathcal{S}$. A unit reward is given only for entering the goal door $D(g)$; reaching any other door also ends the episode but yields no reward.

\textbf{Observation distribution}\hspace{0.5cm} The observations are deterministic. Writing $s_t = (n, p_t, \delta_t, D, g) \in \mathcal{S}$, the observation $o_t = (g_t, c_t, \delta_t) \in \mathcal{O}$ is given by $\Omega(o_t \mid s_t) = 1$ where
$$
g_t =
\begin{cases}
    g + 1 & \text{if } t = 0,\\
    0 & \text{if } t > 0.
\end{cases}
$$
The heading component of the observation is simply $\delta_t$, i.e. the agent heading is fully observed at every timestep. The local window $c_t \in \{0, \ldots, 5\}^{v \times v}$ is centered $d = \lfloor v/2 \rfloor$ cells ahead of the agent along its heading, covering the world cells $\left\{ p_t + d \cdot e(\delta_t) + \omega \mid \omega \in \{-d, \ldots, d\}^2 \right\}$, each labeled by the labeling function $\kappa(\cdot\,; n, D) \colon \mathbb{Z}^2 \to \{0, \ldots, 5\}$,
$$
\kappa(q; n, D) =
\begin{cases}
    2 + i & \text{if } q = D(i) \text{ for some } i \in \{0,1,2,3\},\\
    1 & \text{else if } q_x \leq 0 \text{ or } q_x \geq n-1 \text{ or } q_y \leq 0 \text{ or } q_y \geq n-1,\\
    0 & \text{otherwise},
\end{cases}
$$
i.e. $0$ denotes an empty cell, $1$ a wall, and $2+i$ door $i$. The window is expressed in an \emph{egocentric} frame: it is rotated according to the agent current heading $\delta_t$ so that the agent forward direction always maps to the same fixed array axis, irrespective of its true heading in the room.



\newpage
\section{Training Procedure and Hyperparameters}\label{appendix-hyperparameters}

\textbf{Design of the experiments.} For each environment we use two complementary
approaches to test the connection between DHI and RNN dynamics. In the first, we train
nonlinear RNNs in the usual way and compare performance and generalization with
multistability measured a posteriori, which keeps the evaluation unbiased. We then reproduce
the experiments with minGRU as a parallelizable monostable RNN with transient dynamics and
BMRU as a parallelizable bistable RNN without transient dynamics, which lets us assess the
role of each separately.

\textbf{Network size.} The recurrent layers are small by design: 5 units on the T-maze,
64--128 elsewhere. This is a choice, not a limitation of the analysis. It keeps the dynamical
characterization (\vaat{}, basin structure) interpretable, and it makes the point that
multistability yields DHI at very small cost, whereas fading memory needs substantially
larger networks to buy a longer, and still bounded, effective horizon. Control experiments
with larger networks are reported in \cref{section-mono}.

\textbf{Training behavior.} Models trained on short T-mazes reach an optimal policy
within a few training steps and with almost no variance across seeds. Models trained on
T-mazes of length 100 struggle to learn the task and show high variance; learning curves are
given in \cref{appendix-supFig}. This is expected in an RL setting, since exploration is far
harder in a long corridor, which destabilizes training.

Given that the environments used in this work are partially observable Markov decision processes, the optimal policies and value functions are dependent on the sequence of all observations and actions (\textit{i.e.},  the history). We use the proximal policy optimization (PPO) \citep{schulman2017proximalpolicyoptimizationalgorithms} algorithm, an advantage actor-critic \citep{mnih2016asynchronousmethodsdeepreinforcement} method, to approximate the optimal policy and the value function. The policy and the value function ($\pi$ and VF respectively) are implemented with RNNs. Nothing is shared between the policy and the value function, so the history is encoded separately by both RNNs.

The policy neural network outputs a probability distribution over the action space from which the next action is sampled to encourage exploration at the beginning of the training and exploitation at the end as the distribution entropy decreases.

All models were trained using a modified version of the LSTM PPO implementation from CleanRL \citep{huang2022cleanrl}.

The architecture shown in \cref{fig:architecture1} was considered for the T-maze experiments shown in \cref{fig:tmaze} with the exception of the deep minGRU.
\begin{figure*}[h]
    \centering
     \includegraphics[width=\textwidth]{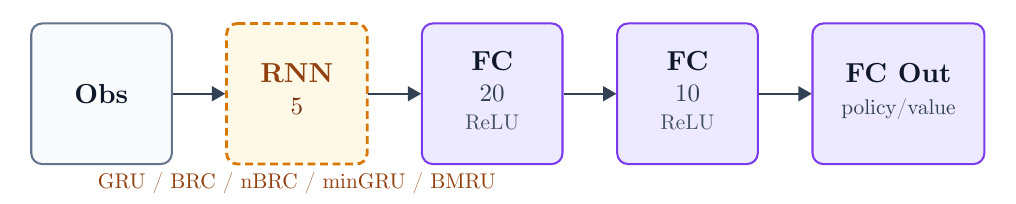}
    \caption{Architecture used for the T-maze experiments shown in \cref{fig:tmaze} (except for the deep minGRU). FC is a fully connected (or dense) layer. RNN is a recurrent layer and is either GRU, BRC, nBRC, minGRU, or BMRU depending on the model. Since the policy and the value function use non-shared distinct networks, the FC Out layer contains 4 neurons for the policy and 1 for the value function.}
    \label{fig:architecture1}
\end{figure*}

For the deep minGRU experiment shown in \cref{fig:tmaze}, the deeper architecture shown in \cref{fig:architecture2} was considered with $N=4$. This architecture was also used for the experiments shown in \cref{fig:tstar} and \cref{fig:benchmarks} with $N=1$ for minGRU and BMRU, and $N=4$ for deep minGRU.
\begin{figure*}[h]
    \centering
     \includegraphics[width=\textwidth]{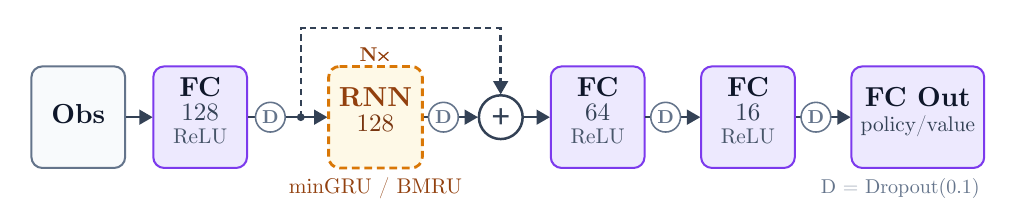}
    \caption{Architecture used for the deep minGRU in \cref{fig:tmaze} and for all models in \cref{fig:tstar} and \cref{fig:benchmarks}. FC is a fully connected (or dense) layer. RNN is $N$ recurrent layers in sequence which are either minGRUs or BMRUs depending on the model, with $N=1$ for minGRU and BMRU, and $N=4$ for deep minGRU. Since the policy and the value function use non-shared distinct networks, the FC Out layer contains various neurons depending on the environment for the policy and 1 for the value function. Each layer is followed by a Dropout layer with probability $p=0.1$.}
    \label{fig:architecture2}
\end{figure*}

The hyperparameters used for each experiment are shown in \cref{table:param}.


\begin{table}[htb]
\centering
\caption{Hyperparameters used during training for the three benchmarks.}
\begin{tabular}{|c|c|c|}
    \hline
    Hyperparameter & T-maze & MiniGrid (Memory / GoToDoor)\\
    \hline
    $\pi$ epochs & $20$ & $10$\\
    VF epochs & $10$ & $10$\\
    $\pi$ optimizer & Adam & Adam\\
    VF optimizer & Adam & Adam\\
    LR scheduler & Cosine annealing & Cosine annealing\\
    & ($T_{max}=250$) & ($T_{max}=250$)\\
    Grad. clip norm & $1$ & $1$\\
    VF coeff. $c_1$ & $1$ & $1$\\
    GAE ($\lambda$) & $0.98$ & $0.98$\\
    Discount ($\gamma$) & $0.998$ & $0.998$\\
    Parallel envs & $50$ & $2048$\\
    Minibatch (N $\times$ size) & $2 \times 25$ & $8 \times 256$\\
    $\text{KL}_{\text{max}}$ early stop & $0.2$ & $0.2$\\
    Sample steps & $1400$ & $100$\\
    Total iterations & $250$ & $250$\\
    \hline
\end{tabular}
\label{table:param}
\end{table}

For the T-maze experiments shown in \cref{fig:tmaze}, the following additional hyperparameters were used: $\text{LR}_\pi = 0.005$, $\text{LR}_\text{VF}=0.005$, PPO clip range $\epsilon=0.2$, Entropy coefficient $c_2=0.01$. For the other experiments, the hyperparameter sweeps are detailed in \cref{appendix-sweep}.

\paragraph{Compute resources} 
All training experiments were performed on an academic HPC cluster, using a single NVIDIA Tesla A100 GPU (40 GB VRAM) per run on a node with AMD EPYC Milan CPUs and up to 512 GB of system RAM. The total compute footprint of the project, including preliminary experiments not reported in this paper, remains modest by modern deep learning standards.

\newpage
\section{Hyperparameter Sweep}
\label{appendix-sweep}

To check that our conclusions do not depend on a particular hyperparameter choice, we performed a random search over the four hyperparameters most likely to affect how much information a recurrent policy retains: the policy and value function learning rates, the PPO clip range $\epsilon$, and the entropy coefficient $c_2$. All other settings were held at the values of \cref{table:param}. The ranges of the hyperparameters are shown in \cref{table:sweep}. For each environment and model, we selected the hyperparameter configuration achieving the highest average reward, breaking ties by the highest average instantaneous reward. The resulting values are shown in \cref{tab:hparams}.

\begin{table}[htb]
\centering
\caption{Random-search space used for the hyperparameter sweep.}
\begin{tabular}{lcc}
    \toprule
    Hyperparameter & Distribution & Range \\
    \midrule
    Learning rate ($\pi$ and VF) & log-uniform & [$10^{-6}$, $10^{-2}$] \\
    Clip range $\epsilon$       & uniform     & [0.1, 0.3] \\
    Entropy coefficient $c_2$    & log-uniform & [$10^{-4}$, $2 \times 10^{-2}$] \\
    \bottomrule
\end{tabular}\label{table:sweep}
\end{table}

\begin{table}[h]
\centering
\caption{Hyperparameters selected via random search for each environment--model pair.}
\label{tab:hparams}
\sisetup{
    scientific-notation = true,
    round-mode = places,
    round-precision = 3,
    table-format = 1.3e-1,
}
\begin{tabular}{ll S S S[scientific-notation=false, round-precision=3, table-format=1.3] S}
\toprule
Environment & Model & {$\pi_\text{lr}$} & {$\text{VF}_\text{lr}$} & {$\epsilon$} & {$c_2$} \\
\midrule
\multirow{3}{*}{Memory}
  & minGRU      & 4.423172359572124e-03  & 2.015116157563954e-03  & 0.13373245644422238 & 1.9304753898697084e-04 \\
  & BMRU        & 1.4080630192920074e-03 & 8.917296737993331e-03  & 0.19056098203581    & 3.4381616799949717e-03 \\
  & Deep minGRU & 9.202389420351202e-04  & 5.584990035728523e-04  & 0.15482153237217836 & 1.2089045935311732e-02 \\
\midrule
\multirow{3}{*}{GoToDoor}
  & minGRU      & 3.708148320059677e-03  & 1.219024909776691e-03  & 0.19846514342046703 & 3.204157330867166e-03 \\
  & BMRU        & 1.423392508732841e-03  & 5.254001945971174e-05  & 0.23295713820330896 & 2.6030464641203387e-04 \\
  & Deep minGRU & 7.312616553289273e-04  & 1.6204570727798644e-03 & 0.247370156226058    & 4.26120921684252e-03 \\
\bottomrule
\end{tabular}
\end{table}

The outcome is the one predicted by \cref{thm:mono}: hyperparameter optimization changes
where a monostable model stops generalizing, but not whether it stops. No sampled
configuration produced a monostable model with an unbounded effective horizon, which is
expected, since monostability is a property of the update equation and not of the training
procedure.

\newpage
\section{Measuring the Effective Horizon}
\label{appendix-teff}

\Cref{def:teff} defines the effective horizon $T^\star$ of a policy as the largest horizon up to which it remains optimal. Since returns are stochastic and optimality is not exactly attained in practice, the following empirical surrogate can be used. Let $R(T)$ be the mean return of a policy evaluated at horizon $T$ over $n$ episodes, $R_{\mathrm{opt}}$ the optimal return, and $R_{\mathrm{rand}}$ the return of a uniformly random decision at the junction.
We define the normalized score
\[
  \hat{R}(T) = \frac{R(T) - R_{\mathrm{rand}}}{R_{\mathrm{opt}} - R_{\mathrm{rand}}}
  \in (-\infty, 1],
\]
and set
\[
  T^\star = \max\left\{ T \in \mathcal{G} \mid \hat{R}(T') \geq 1 - \kappa
   \text{ for all } T' \in \mathcal{G}, T' \leq T \right\},
\]
where $\mathcal{G}$ is the evaluated grid of horizons and $\kappa$ a tolerance.
A model whose score never drops below the threshold on the whole evaluated grid is reported as $T^\star > \max \mathcal{G}$; this is a lower bound, and is the expected outcome for a model achieving DHI, since by \cref{thm:multi} its score is horizon-independent.

For the experiments shown in \cref{fig:tstar}, we evaluated the policies at every horizon in $\mathcal{G}=\{1,\ldots,10^{6}\}$ with $\kappa=0$. Each policy was run for 2 episodes, one per goal direction.


\newpage
\section{Computation of the VAA}
\label{appendix-vaa}

\paragraph{Classifying trained models.} Architecture fixes the dynamical regime for affine
and input-gated cells, but a trained nonlinear cell must be classified after the fact. We do
so with the \vaat{}: a value at the bottom of the range indicates a single attractor, that
is, monostability, and a value at the top indicates that distinct initial conditions settle
on distinct attractors. The monostable population reported in \cref{section-mono} is
therefore of two kinds, minGRU, monostable by construction, and those GRU, BRC and nBRC seeds
whose trained parameters land in the monostable regime; the multistable population of
\cref{section-multi} is the mirror image, with BMRU multistable by construction.


The variability among attractors (VAA) quantifies the level of multistability of an RNN $\Phi$ by measuring the number of different stable states reached by this model from a set of initial states $\mathcal{H}_0$ under a constant input $x$. It returns the normalized number of different stable states, \textit{i.e.},  a scalar between $\frac{1}{|\mathcal{H}_0|}$, when there is only one stable state, and $1$, when all stable states are different. Mathematically, the VAA is defined as:
\begin{equation*}
    \text{VAA}(\Phi, \mathcal{H}_0, x) = \frac{1}{|\mathcal{H}_0|} \sum_{i = 1}^{|\mathcal{H}_0|} \frac{1}{\sum_{j=1}^{|\mathcal{H}_0|} \mathbf{1} \left( \limsup_{n \to \infty} \| \Phi_x^n(h_i) - \Phi_x^n(h_j) \| = 0 \right)},
\end{equation*}
where $\mathbf{1}(\cdot)$ is the indicator function and $\Phi_x^n$ is the $n$-fold composition of $\Phi_x = \Phi(\cdot, x)$.

However, the rigorous definition of the VAA implies computing limits to infinity, making it unusable in practice. Therefore, an approximation of the VAA is used in the experiments. This approximation introduces two parameters: $M$, the number of iterations of $\Phi_x$, and $\epsilon$, the tolerance to consider two states as identical. Mathematically, it gives:
\begin{equation*}
    \text{VAA}(\Phi, \mathcal{H}_0, x; M, \epsilon) = \frac{1}{|\mathcal{H}_0|} \sum_{i = 1}^{|\mathcal{H}_0|} \frac{1}{\sum_{j=1}^{|\mathcal{H}_0|} \mathbf{1} \left( \| \Phi_x^M(h_i) - \Phi_x^M(h_j) \| \leq \epsilon \right)}.
\end{equation*}

When computing VAA to a specific task, the set of initial states $\mathcal{H}_0$ and the constant input $x$ have to be defined. In this work, we consider the classic T-maze task. For the classic T-maze, we simply used the hidden states obtained at the first timestep after giving the position of the reward to the agent. As only two positions are allowed, $|\mathcal{H}_0| = 2$. The constant input $x$ is simply the observation given to the agent while it has not reached the junction. In this work, we used $M=3000$ and $\epsilon=10^{-5}$, values large and small enough, respectively, that the convergence analysis of the hidden state was unaffected by further adjustment.

\newpage
\section{Learning curves and RNN dynamics}
\label{appendix-supFig}

Once horizon independence is secured, the training horizon can be chosen for convenience rather than for coverage, and a short one is the better choice: it is cheaper, and it avoids the slow transients that a multistable model trained at $T = 100$ develops and never needs, which then cost it accuracy on the short settings. This appendix reports this comparison.

\begin{figure}[h]
  \centering
  \includegraphics[width=0.55\textwidth]{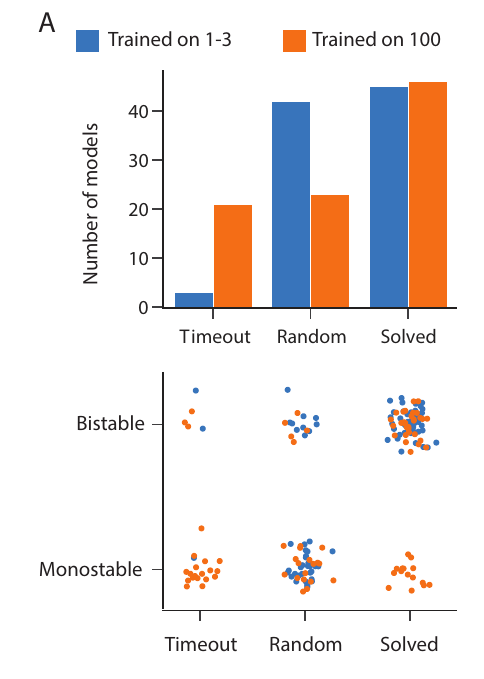}
  \caption{\textbf{Training horizon and the incidence of multistability.} Top, number of
  models that timed out, reached the end but acted randomly, or solved the T-maze of size
  100. Bottom, the same models separated by their performance and by their
  monostable/bistable character. Blue, models trained on T-mazes of size 1--3; orange,
  models trained on T-mazes of size 100.}
  \label{fig:incidence}
\end{figure}

\begin{figure}[h]
    \centering
    \includegraphics[width=0.9\textwidth]{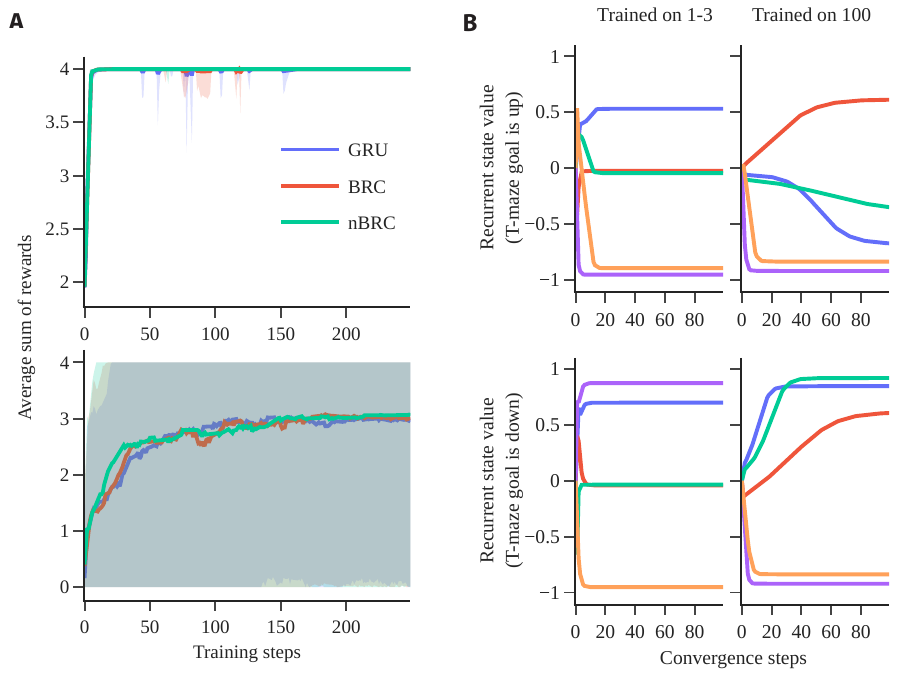}
    \caption{\textbf{A.} Evolution of the average sum of rewards (with $2\sigma$) obtained by the agents during their training on T-mazes of lengths 1-3 (top), and of length 100 (bottom). \textbf{B.} Time evolution of the internal states of two bistable models, one trained on small T-mazes (left), the other trained on long T-mazes (right), for the two possible cases on a T-maze of length 100. The figure shows that models trained on long T-mazes have slower transient dynamics before reaching equilibrium.}
    \label{figsup:LC}
\end{figure}

\textbf{Training horizon and the incidence of multistability.} \Cref{fig:incidence} (bottom)
categorizes each model by its performance on the long T-maze (timeout, random or solved) and
by its stability properties. For models trained on short T-mazes, the correlation between
multistability and solving the environment is clear: no monostable model trained at
$T \in [1,3]$ solves $T = 100$, whereas most bistable ones do. Models trained directly at
$T = 100$ solve it using either slow transient dynamics or multistability, and fewer of them
reach bistability at all (36/90, against 57/90 at the short horizon), indicating that
training on long horizons does not itself promote encoding through multistability.

\textbf{Transient dynamics and the training horizon.}
Multistable models trained at the two horizons behave differently at the short end of the
evaluation range (\cref{fig:tmaze}A, bottom). Models trained at $T \in [1,3]$ are optimal
over the whole range. Models trained at $T = 100$ remain optimal at long horizons, as
\cref{thm:multi} requires, but lose accuracy on the shortest mazes. The reason is that even a
multistable cell needs transient dynamics to settle into a basin, and those transients are
considerably slower when the model is trained at a long horizon, where nothing penalizes a slow
approach to equilibrium (\cref{figsup:LC}B). At short horizons, the decision step then arrives before the state has settled, in a regime the model never saw during training, and the readout misfires. Training at the short horizon avoids the problem entirely: the transients are fast because they had to be, and the multistability that carries the cue is unaffected by the horizon at which it was learned.


\clearpage
\section{Equations for the different RNNs used in this work}
\label{appendix-RNNeq}

The equations for the different models used in this paper are provided below. In each equation, $W$ are learnable weight matrices, $b$ and $w$ are learnable parameter vectors, $x_t$ is the input, $h_t$ is the state, and $r_t$, $z_t$ and $n_t$ are gate variables. $\text{H}(\cdot)$ denotes the Heaviside step function, and $\text{S}(\cdot)$ is the sign function. 

\begin{align*}
    \intertext{\subsection{GRU \citep{cho2014learningphraserepresentationsusing}}}
    r_t &= \sigma(W_{xr} x_t + W_{hr} h_{t-1} + b_{r})\\
    z_t &= \sigma(W_{xz} x_t + W_{hz} h_{t-1} + b_{z})\\
    n_t &= \tanh(W_{xn} x_t + r_t \odot W_{hn} h_{t-1} + b_{n})\\
    h_t &= z_t \odot h_{t-1} + (1 - z_t) \odot n_t
    \intertext{\subsection{BRC \citep{vecoven2021bioinspiredbistablerecurrentcell}}}
    r_t &= 1 + \tanh(W_r x_t + w_r \odot h_{t-1} +b_r)\\
    z_t &= \sigma(W_z x_t + w_z \odot h_{t-1} + b_z)\\
    n_t &= \tanh(W_n x_t + r_t \odot h_{t-1} + b_n)\\
    h_t &= z_t \odot h_{t-1} + (1 - z_t) \odot n_t
    \intertext{\subsection{nBRC \citep{vecoven2021bioinspiredbistablerecurrentcell}}}
    r_t &= 1 + \tanh(W_{xr} x_t + W_{hr} h_{t-1}+b_r)\\
    z_t &= \sigma(W_{xz} x_t + W_{hz} h_{t-1}+b_z)\\
    n_t &= \tanh(W_n x_t + r_t \odot h_{t-1}+b_n)\\
    h_t &= z_t \odot h_{t-1} + (1 - z_t) \odot n_t
    \intertext{\subsection{minGRU \citep{feng2024weregrudminimalgatedrecurrent}} }
    z_t &= \sigma(W_z x_t + b_z)\\
    n_t &= W_n x_t + b_n\\
    h_t &= z_t \odot h_{t-1} + (1 - z_t) \odot n_t
    \intertext{\subsection{BMRU \citep{degeeter2026parallelizable}}}
    n_t &= W_n x_t + b_n\\
    \beta_t &= \left|W_\beta x_t + b_\beta\right|\\
    z_t &= \text{H}(|n_t| - \beta_t)\\
    h_t &= z_t \odot \text{S}(n_t) \odot \alpha + (1 - z_t) \odot h_{t-1}
\end{align*}

\begin{figure}[h]
    \centering
    \includegraphics[width=0.9\textwidth]{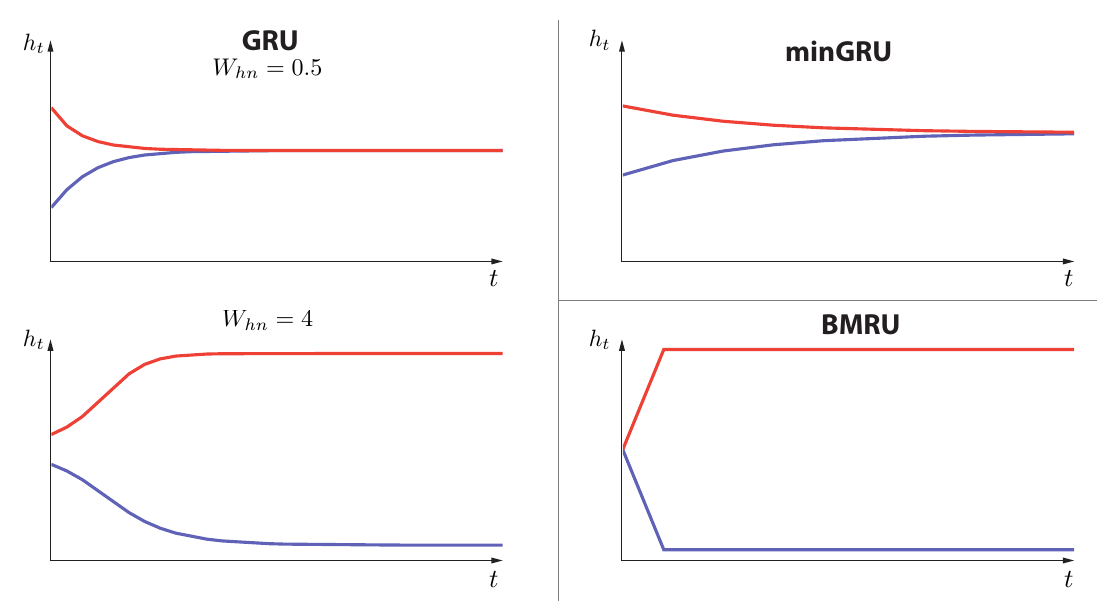}
    \caption{Left, time evolution of the internal state of a single GRU for two different initial conditions (red and blue curves) and two different values of $W_{hn}$. The top graph shows a monostable behavior: both initial conditions reach the same equilibrium. The bottom graph shows a bistable behavior: each initial condition reaches a different equilibrium. Right, time evolution of the internal state of a single minGRU (top) and BMRU (bottom) for two different initial conditions (red and blue curves). minGRU exhibits monostable behavior with transient dynamics, whereas BMRU exhibit bistable behavior with a one step convergence.}
    \label{figsup:RNN}
\end{figure}


\end{document}